\documentclass{article}

\usepackage[preprint]{main}

\usepackage[utf8]{inputenc} % allow utf-8 input
\usepackage[T1]{fontenc}    % use 8-bit T1 fonts
\usepackage{hyperref}       % hyperlinks
\usepackage{url}            % simple URL typesetting
\usepackage{booktabs}       % professional-quality tables
\usepackage{amsfonts}       % blackboard math symbols
\usepackage{nicefrac}       % compact symbols for 1/2, etc.
\usepackage{microtype}      % microtypography
\usepackage{xcolor}         % colors

\usepackage{graphicx}
\usepackage{subfigure}
\usepackage{wrapfig}

\usepackage{amsmath} 
\usepackage{algorithm}
\usepackage{algpseudocode}
\usepackage{hyperref}
\usepackage[nameinlink,capitalize]{cleveref}

\usepackage{bbm}
\usepackage{amsmath}
\usepackage{amssymb}
\usepackage{mathtools}
\usepackage{amsthm}
\usepackage{mathrsfs}
\usepackage{microtype}

\theoremstyle{plain}
\newtheorem{theorem}{Theorem}[section]
\newtheorem{proposition}[theorem]{Proposition}
\newtheorem{lemma}[theorem]{Lemma}
\newtheorem{corollary}[theorem]{Corollary}
\theoremstyle{definition}
\newtheorem{definition}[theorem]{Definition}
\newtheorem{assumption}[theorem]{Assumption}
\theoremstyle{remark}
\newtheorem{remark}[theorem]{Remark}

\usepackage{multirow}

\title{Matrix-Space Reinforcement Learning for Reusing Local Transition Geometry}

\author{%
    Zuyuan Zhang\\
    The George Washington University\\
    \texttt{zuyuan.zhang@gwu.edu}\\
    \And
    Carlee Joe-Wong\\
    Carnegie Mellon University\\
    \texttt{cjoewong@andrew.cmu.edu}\\
    \And
    Tian Lan\\
    The George Washington University\\
    \texttt{tlan@gwu.edu}
}

\begin{document}

\maketitle

\begin{abstract}
Compositional generalization in sequential decision-making requires identifying which parts of prior rollouts remain useful for new tasks. Existing methods reuse skills or predictive models, but often overlook rich local transition geometry and dynamics. We propose Matrix-Space Reinforcement Learning (MSRL), a geometric abstraction that represents trajectory segments through positive semidefinite matrix descriptors aggregating first- and second-order statistics of lifted one-step transitions. These descriptors expose shared hidden structure, support algebraic composition in an abstract matrix space, and reveal opportunities for transfer. We prove that the descriptor is well defined up to coordinate gauge, complete for the induced low-order additive signal class, additive under valid segment composition, and minimally sufficient among admissible additive descriptors. We further show that conditioning value functions on the trajectory-segment matrix yields a first-order smooth approximation of action values, enabling source-learned matrix-to-value mappings to bootstrap learning in new tasks. MSRL is plug-in compatible with standard model-free and model-based methods, while obstruction filtering rejects implausible compositions. Empirically, MSRL achieves the best average finite-budget target AUC of 0.73, outperforming MSRL from scratch (0.65), TD-MPC-PT+FT (0.63), and TD-MPC (0.57).

\end{abstract}

\section{Introduction}

Sequential decision making is formalized through states, actions, rewards, transition laws, and Bellman recursions \citep{bellman1966dynamic,puterman2014markov,sutton1998reinforcement}. Modern reinforcement learning released the potential of this formalism through highly-capable value and policy models that can often learn highly complex policies \citep{watkins1992q,mnih2015human,schulman2017proximal,haarnoja2018soft,zhang2026cochain}. Reusing components/knowledge across related tasks, however, remains difficult: even when two tasks share similar local geometry or dynamics, it is not obvious how to identify what underlying structures could enable generalization, how to leverage trajectories/demonstrations of different appearances, or how to harness them to improve value/policy learning without incurring infeasible behaviors in decision-making \citep{taylor2009transfer,lazaric2012transfer,rusu2016progressive,zhang2026operator,zhang2026structuring}. Existing methods reuse temporally extended skills, predictive occupancies, goal-conditioned values, or behaviorally compressed states \citep{sutton1999between,ravindran2003smdp,barreto2019option,dayan1993improving,barreto2017successor,barreto2018transfer,machado2017eigenoption,schaul2015universal,andrychowicz2017hindsight,ferns2004metrics,givan2003equivalence,li2006towards,abel2016near,castro2020scalable,zhang2020learning,zhang2026geometry,zhang2025learning}, but they usually do not leverage rich local geometry and dynamics (such as transitions, structures, and replay segments) to identify opportunities for compositional generalization---the ability of an AI/ML agent to
generate novel, complex combinations of known components/knowledge and translate them to new settings.

This paper proposes a novel geometric abstraction to quantify local transition dynamics/geometry and to enable efficient compositional generalization in sequential decision-making. We view a trajectory segment as a finite consecutive block of transitions from a rollout, where segments may share common prefixes, suffixes, full rollouts, or partial overlaps. This allows us to describe trajectories/demonstrations of different appearances (e.g., horizon length, task realizations, and action/space representations) through a common set of geometric descriptors within one unified space. These descriptors make shared hidden structures explicit, distinguish algebraic composition in an abstract representation space, and reveal opportunities for compositional generalization. We detail our \textbf{four contributions} below.

\textbf{\textit{(1) We introduce the trajectory-segment matrix \(M(\tau)\) as a geometric descriptor}} (Section~\ref{sec:traj-matrix}) to capture sufficient local geometry and dynamics information to enable generalizability. Given a trajectory segment
$
\tau=((o_t,a_t,r_t,o_{t+1}))_{t=0}^{T-1},
$
where \(o_t\) and \(o_{t+1}\) are consecutive observations, \(a_t\) is the action, \(r_t\) is the reward, and \(T\) is the segment length, we define a fixed step lift \(\psi_t(\tau)\) that records the embedded transition midpoint, directed displacement, action-conditioned midpoints and displacements, reward, and a constant channel. We then define
$
M(\tau)=\sum_{t=0}^{T-1}\psi_t(\tau)\psi_t(\tau)^\top,
$
a covariance-like matrix aggregating second-order statistics of lifted one-step transitions in a positive semidefinite cone \citep{tuzel2006region,pennec2006riemannian,muandet2017kernel,johnson1970positive}.

The goal of this abstraction is to capture critical local geometry and dynamics that induce compositional generalization structures into value learning, while capturing only minimal information needed for sequential decision-making. More precisely, the abstraction is chosen to retain the low-order transition information used by the value learning: where transitions occur, which directions and actions induce movement, how rewards align with them, and how first-order totals accumulate \citep{kondor2002diffusion,gartner2003graph,mahadevan2007proto}. Because \(M(\tau)\) is a commutative sum, it does not encode the full chronological sequence \citep{graves2012long,vaswani2017attention} which often has different appearances across tasks or environments. The benefit of our proposed abstraction is a transparent algebra: harnessing compatible segments in compositional generalization becomes matrix addition (while preserving structure of the abstraction), and overlapping local restrictions can be assembled by inclusion--exclusion operators.

\textbf{\textit{(2) We prove a set of key foundational properties of the proposed and geometric descriptor}} (Section~\ref{sec:traj-matrix-theory}). We prove that \(M(\tau)\) is well defined up to coordinate gauge, complete for the induced low-order additive signal class, additive under valid segment composition, recoverable from overlapping restrictions, and separable from environment-level realizability through an obstruction score. We also prove that \(M(\tau)\) is minimally sufficient among admissible additive descriptors. These results establish an algebra on the \(M(\tau)\) construction and justify the use of the proposed abstraction for transferring and learning value functions in sequential decision-making.

\textbf{\textit{(3) We develop a matrix-space reinforcement learning (MSRL) theory}} (Section~\ref{sec:matrix-rl}) by first proving a matrix-conditioned value law and local smoothness for learning the value function in Section~\ref{sec:matrix-value}. We show that conditioning on the underlying matrix abstraction provides a first order, smooth approximation of the action-value function: $
\bar Q^\star(o,M_0+M_1,a)
\approx
\bar Q^\star(o,M_0,a)+\bar Q^\star(o,M_1,a)-\bar Q^\star(o,0,a).
$ for compositional matrix perturbations $M_0$ and $M_1$. This establishes the foundation for generalization using value functions learned on existing geometry and dynamics. MSRL pretrains an \(M\!\to\!V\) branch on sources, initializes the target critic, and learns residual task-specific value from target data. Our method is compatible with standard model-free and model-based baselines, where a matrix-to-value mapping learned on source environments bootstraps value-function learning in new tasks 
%initializes target-side training 
\citep{fujimoto2018addressing,janner2019trust,hafner2025mastering,hansen2022temporal,taylor2009transfer,lazaric2012transfer,rusu2016progressive}.

\textbf{\textit{(4) We validate our MSRL method}} using source-side structural diagnostics and source-initialized transfer from simple environments to complex target tasks. MSRL consistently shows higher few-shot AUC than both vanilla RL and transfer learning baselines under matched target interaction budgets, demonstrating its enabling of compositional generalization.

We finally conclude with a discussion of limitations in Section~\ref{sec:conclusion}.

\section{Matrix Proxy for Sampled Compositional Structure}
\label{sec:traj-matrix}
This section defines the descriptor used throughout the paper. The goal is to turn observed rollout data into a fixed-dimensional object that can be stored, compared, composed with other segments, and supplied to a matrix-conditioned value model. A trajectory segment is any finite consecutive block of transitions from a rollout; different segments may have different lengths and may overlap when extracted from the same trajectory. Once a common step lift is fixed, every such segment is mapped to one matrix in a shared positive semidefinite cone.

\paragraph{Fixed step lift.}
For each environment \(e\), raw observations and actions may first be mapped by environment-specific adapters into common spaces \(\mathcal O\) and \(\mathcal A\). We call a family of environments \emph{aligned} when these adapters are fixed and all matrices that are compared, composed, or transferred use the same embedded observation map, action map, dimensions, and block order. This is an assumption on the transfer setting, not a claim that arbitrary environments are automatically aligned. Fix
$
\phi:\mathcal O\to\mathbb R^d,
\qquad
\xi:\mathcal A\to\mathbb R^q,
$
for the aligned family under study.

For a finite observable trajectory segment
$
\tau=((o_t,a_t,r_t,o_{t+1}))_{t=0}^{T-1},
$
the local index \(t\) ranges over the \(T\) transitions inside this segment; \(o_t\) and \(o_{t+1}\) are consecutive observations, \(a_t\) is the action applied at the transition, and \(r_t\) is the observed reward. Define
$
x_t:=\phi(o_t),\quad
x_t^+:=\phi(o_{t+1}),\quad
c_t:=\frac{x_t+x_t^+}{2},\quad
\delta_t:=x_t^+-x_t.
$
Here \(c_t\) is the midpoint of the embedded one-step transition, and \(\delta_t\) is its directed displacement. To place local position, directed displacement, action dependence, reward, and additive first-order statistics in one common space, define the segment-dependent lifted vector
$
\psi_t(\tau)
:=
\begin{bmatrix}
c_t^\top &
\delta_t^\top &
\bigl(\xi(a_t)\otimes c_t\bigr)^\top &
\bigl(\xi(a_t)\otimes \delta_t\bigr)^\top &
r_t &
1
\end{bmatrix}^{\top}
\in \mathbb{R}^{m},
$
where
$m:=2d+2qd+2.
$
All lifted quantities are assumed finite throughout the definition.

\begin{definition}[Matrix proxy induced by a trajectory segment]
\label{def:traj-matrix-final}
The matrix proxy of \(\tau\) is
$
M(\tau):=\sum_{t=0}^{T-1}\psi_t(\tau)\psi_t(\tau)^\top \in \mathbb S_+^m.
$
\end{definition}

The outer-product aggregation can be viewed as a finite-dimensional PSD descriptor of lifted local transitions; PSD matrix geometry, covariance descriptors, kernel embeddings, and graph kernels provide useful background for this representation choice \citep{johnson1970positive,tuzel2006region,muandet2017kernel,vishwanathan2010graph,shervashidze2011weisfeiler}. The role here is operational: segments, prefixes, and full rollouts all have the same matrix dimension, so a replay library can store them as reusable increments and a single value model can evaluate accumulated matrix states.

Because \(M(\tau)\) is a commutative sum of one-step outer products, it records low-order additive statistics of the lifted transitions rather than the full chronological sequence. It may distinguish a segment from its time reversal when the lifted transition vectors change, but it does not distinguish two segments that contain the same multiset of lifted one-step vectors in a different order. Symbolic transition-pattern examples are deferred to Appendix~\ref{app:symbolic-subgraph-patterns}; blockwise interpretation is deferred to Appendix~\ref{app:matrix-blocks}.

\section{Structural Guarantees of the Matrix Proxy}
\label{sec:traj-matrix-theory}

The six properties below are the requirements needed later for matrix-space control. Gauge covariance makes comparisons meaningful under aligned coordinate changes. Completeness specifies exactly which low-order additive signals are retained. Composition and gluing justify reusing stored local segments in a common matrix algebra. Minimal sufficiency states that, within the admissible additive class, no additional history representation is needed to recover those signals. Obstruction separates formal matrix compositions from candidates that are realizable in the environment.

\subsection{Gauge covariance and well-definedness}

\begin{theorem}[Well-definedness up to gauge]
\label{thm:traj-well-defined}
For every finite trajectory segment satisfying Definition~\ref{def:traj-matrix-final}, \(M(\tau)\) is well defined, symmetric, and positive semidefinite. Moreover, if
$
\widetilde \phi(o)=S\phi(o),\qquad \widetilde \xi(a)=U\xi(a),
$
with \(S\in GL(d)\) and \(U\in GL(q)\), then
$
\widetilde M(\tau)=H M(\tau)H^\top,
\qquad
H:=\operatorname{BlkDiag}(S,S,U\otimes S,U\otimes S,1,1).
$
Hence the intrinsic descriptor is the congruence class \([M(\tau)]\), rather than any particular coordinate realization.
\end{theorem}

This gauge-covariant reading is consistent with geometric representation-learning principles in which the coordinate realization is less important than the transformation law of the descriptor \citep{bronstein2021geometric}.

The next definition fixes the information claimed to be preserved. The class is intentionally low-order: it contains the additive first-order and quadratic signals exposed by the lift, which are the signals used by the matrix-conditioned value interface below.

\subsection{Completeness for the chosen low-order class}

\begin{definition}[Low-order additive compositional class]
\label{def:qagc}
Write
$
\bar\psi_t
:=
\begin{bmatrix}
c_t^\top &
\delta_t^\top &
\bigl(\xi(a_t)\otimes c_t\bigr)^\top &
\bigl(\xi(a_t)\otimes \delta_t\bigr)^\top &
r_t
\end{bmatrix}^{\top}
\in \mathbb{R}^{m-1},
\qquad
\psi_t
=
\begin{bmatrix}
\bar\psi_t^\top & 1
\end{bmatrix}^{\top}.
$
Let \(\mathfrak F_{\mathrm{aq}}\) be the class of trajectory-segment functionals \(F\) for which there exist \(a\in\mathbb R\), \(b\in\mathbb R^{m-1}\), and \(C\in\mathbb S^{m-1}\) such that
$
F(\tau)
=
\sum_{t=0}^{T-1}
\Bigl(
a+b^\top \bar\psi_t+\bar\psi_t^\top C\bar\psi_t
\Bigr).
$
Thus \(\mathfrak F_{\mathrm{aq}}\) contains exactly the additive first-order and quadratic statistics induced by the fixed lift. These are the low-order signals of sampled compositional structure preserved by the descriptor.
\end{definition}

\begin{theorem}[Completeness]
\label{thm:traj-complete}
For every \(F\in\mathfrak F_{\mathrm{aq}}\), there exists \(A\in\mathbb S^m\) such that
$
F(\tau)=\langle A,M(\tau)\rangle
\qquad
\text{for every finite trajectory segment }\tau.
$
Conversely, every \(A\in\mathbb S^m\) defines a functional in \(\mathfrak F_{\mathrm{aq}}\) through the same formula. Consequently,
$
M(\tau)=M(\tau')
\quad\Longleftrightarrow\quad
F(\tau)=F(\tau')
\ \ \text{for all }F\in\mathfrak F_{\mathrm{aq}}.
$
Equivalently, \(M(\tau)\) captures exactly the chosen low-order additive information induced by the fixed lift.
\end{theorem}

The next two results give the representation-level algebra. A valid concatenation \(\tau=\tau^{(1)}\star\tau^{(2)}\) means that \(\tau^{(1)}\) and \(\tau^{(2)}\) are consecutive segment blocks from the same rollout, or from compatible rollouts whose terminal and initial observations match under the aligned representation. Such concatenation becomes matrix addition, while overlapping local restrictions combine through inclusion--exclusion in the same ambient space.

\subsection{Composition and gluing}

\begin{theorem}[Composition]
\label{thm:traj-composition}
If
$
\tau=\tau^{(1)}\star \tau^{(2)}
$
is a valid concatenation, then
$
M(\tau)=M(\tau^{(1)})+M(\tau^{(2)}).
$
Therefore concatenated segments induce composed descriptors in the same matrix space.
\end{theorem}

Each local segment contributes one matrix increment interpreted as sampled local structure. The theorem says that compositional assembly is linear in matrix space even though the raw observations arrive in sequential form.

\begin{theorem}[Local-to-global gluing]
\label{thm:traj-gluing}
Let \(\{I_k\}_{k=1}^K\) be an interval cover of \(\{0,\dots,T-1\}\). For each nonempty \(S\subseteq\{1,\dots,K\}\), write
$
I_S:=\bigcap_{k\in S} I_k,
\qquad
\tau_S:=\tau|_{I_S}
\quad\text{when }I_S\neq\varnothing.
$
Here \(\tau|_{I_S}\) denotes the subsegment containing exactly the transitions whose local indices lie in \(I_S\). Let
$
\mathcal N:=\{S\subseteq\{1,\dots,K\}: S\neq\varnothing,\ I_S\neq\varnothing\}.
$
Then
$
M(\tau)
=
\sum_{S\in\mathcal N}
(-1)^{|S|+1} M(\tau_S).
$
Hence the global matrix is recovered exactly from overlapping local restrictions by inclusion--exclusion: matrices from individual intervals include all local transitions, and matrices from nonempty intersections subtract the duplicate transition contributions introduced by overlaps.

More generally, any family of matrices \(\{M_S\}_{S\in\mathcal N}\) indexed by the same nonempty overlaps defines a unique candidate glued matrix
$
\widehat M
:=
\sum_{S\in\mathcal N}
(-1)^{|S|+1} M_S
\in\mathbb S^m.
$
This is a representation-level local-to-global construction in the same ambient dimension. Whether \(\widehat M\) is realizable by an actual trajectory in a given environment is addressed separately by Theorem~\ref{thm:traj-obstruction}.
\end{theorem}

The gluing theorem formalizes the local-to-global interpretation used by replay libraries: the learner may store windows, prefixes, or partially overlapping fragments rather than a single complete task-level trajectory, and the matrix algebra specifies how to assemble their nonduplicated contributions.

The next result concerns minimal sufficiency within the admissible additive class. This is the relevant comparison class for the method because the matrix-conditioned value interface and segment library use additive first-order and quadratic statistics of the fixed lift, not unrestricted history representations.

\subsection{Minimal sufficiency}

\begin{definition}[Admissible additive descriptor]
\label{def:admissible-descriptor}
A descriptor \(D\) is called admissible if, for some fixed \(p\in\mathbb N\), it maps each finite trajectory segment \(\tau\) to a point \(D(\tau)\in\mathbb R^p\), is additive under valid concatenation, and is sufficient for all functionals in \(\mathfrak F_{\mathrm{aq}}\) in the following factorization sense: for every \(F\in\mathfrak F_{\mathrm{aq}}\), there exists a Borel measurable map \(h_F:\mathbb R^p\to\mathbb R\) such that
$
F(\tau)=h_F(D(\tau))
\qquad
\text{for every finite trajectory segment }\tau .
$
\end{definition}

\begin{theorem}[Minimal sufficiency]
\label{thm:traj-minimal}
The matrix proxy \(M\) is minimally sufficient within the class of admissible additive descriptors: if \(D\) is any admissible descriptor, then there exists a measurable map \(\rho:\mathbb R^p\to\mathbb S^m\) such that
$
M(\tau)=\rho(D(\tau))
\quad
\text{for every }\tau.
$
Hence \(M\) is the coarsest admissible additive descriptor for the class \(\mathfrak F_{\mathrm{aq}}\), up to measurable post-processing.
\end{theorem}

The final result adds the environment-specific layer. Composition and gluing are closed at the level of descriptor algebra, but realizability depends on the actual environment.

\subsection{Realizability and obstruction}

\begin{definition}[Reachable matrix set]
\label{def:reachable-matrix-set}
For an environment \(e\) and horizon \(T\), define the reachable matrix set
$
\mathcal R_e^{(T)}
:=
\{M(\tau): \tau \text{ is realizable in } e \text{ with horizon } T\}.
$
\end{definition}

\begin{definition}[Obstruction score]
\label{def:obstruction}
For a candidate matrix \(\widehat M\in\mathbb S^m\), define
$
\operatorname{Obs}_e^{(T)}(\widehat M)
:=
\inf_{M\in\mathcal R_e^{(T)}} \|\widehat M-M\|_F .
$
\end{definition}

\begin{theorem}[Obstruction]
\label{thm:traj-obstruction}
Let \(\widehat M\) be a matrix obtained by composition or gluing of local matrices, and let \(T\) be the candidate horizon encoded by its constant-channel block, \(\widehat M_{\mathbf 1\mathbf 1}=T\). If
$
\operatorname{Obs}_e^{(T)}(\widehat M)>0,
$
then \(\widehat M\) does not correspond to any realizable trajectory in environment \(e\) at horizon \(T\). Therefore positive obstruction is a certificate of inconsistency with the environment.
\end{theorem}

No converse is claimed without additional closure assumptions on \(\mathcal R_e^{(T)}\). The theorem separates two levels of validity: algebraic validity in matrix space and realizability in the actual sequential decision environment.

\section{Value as a Low-Order Signal on Matrix Space}
\label{sec:matrix-value}

The structural theorems above explain why \(M(\tau)\) is compositional and sufficient for a clear low-order information class. The control interface additionally requires a relationship between the matrix descriptor and value, connecting the representation to the Bellman view of sequential decision making \citep{bellman1966dynamic,puterman2014markov,sutton1998reinforcement}.
Theorem~\ref{thm:traj-complete} and Appendix~\ref{app:symbolic-subgraph-patterns} show that \(M\) exactly retains the lift-induced low-order structural quantities, including reversal parity, path endpoint displacement, cycle closure, and oriented circulation. We do not claim that \(M\) preserves the full topology of the grounded transition graph; the preserved content is the matrix-visible structural information induced by the common lift. The remaining question for control is whether those preserved directions are sufficient for Bellman recursion and local value variation.

\begin{assumption}[Matrix sufficiency for optimal value]
\label{ass:matrix-conditioned-value}
For each environment \(e\) in the aligned family, let
\(Q_e^\star(o,h,a)\) denote the optimal action-value function when the
current observation is \(o\), the past rollout prefix is \(h\), and the
candidate action is \(a\). There exists a measurable function
$
\bar Q_e^\star:\mathcal O\times\mathbb S_+^m\times\mathcal A\to\mathbb R
$
such that, for every reachable rollout prefix
\(h_t=\tau_{0:t-1}\) and every action \(a\),
$
Q_e^\star(o_t,h_t,a)
=
\bar Q_e^\star\!\left(o_t,M(h_t),a\right).
$
Equivalently, any two reachable prefixes with the same current
observation and the same accumulated matrix state induce the same
optimal action values.
\end{assumption}

Assumption~\ref{ass:matrix-conditioned-value} is not an existence claim
for an arbitrary auxiliary function. It is a sufficiency hypothesis for
the optimal value: after conditioning on the current observation, the
history dependence that is relevant to optimal control is assumed to
enter through the accumulated matrix state \(M(h_t)\). This need not hold
for arbitrary partially observed or nonstationary environments. It
specifies the regime in which the matrix descriptor is intended to act
as a control-relevant state abstraction, analogous in role to
goal-conditioned value parameterizations, successor-feature
decompositions, and bisimulation-based state abstractions
\citep{schaul2015universal,barreto2017successor,barreto2018transfer,
ferns2004metrics,zhang2020learning,gelada2019deepmdp}.

\begin{assumption}[Lifted Bellman compatibility]
\label{ass:matrix-bellman-compatibility}
For each environment \(e\), the lifted pair \((o,M)\) admits a
well-defined one-step transition kernel. That is, for every reachable
\((o,M)\) and action \(a\), there exists a kernel
$
P_e(d r,d o'\mid o,M,a)
$
such that the conditional law of the next reward and observation,
given any rollout prefix \(h_t\) satisfying
\((o_t,M(h_t))=(o,M)\), is \(P_e(\cdot\mid o,M,a)\). With the one-step
lift denoted by \(\psi(o,a,r,o')\) and
$
M^+ = M+\psi(o,a,r,o')\psi(o,a,r,o')^\top,
$
the matrix-conditioned optimal value satisfies
$
\bar Q_e^\star(o,M,a)
=
\mathbb E_{(r,o')\sim P_e(\cdot\mid o,M,a)}
\left[
r+\gamma \sup_{a'\in\mathcal A}
\bar Q_e^\star(o',M^+,a')
\right],
$
for a discount factor \(\gamma\in[0,1)\), whenever the right-hand side is
well defined.
\end{assumption}

Assumption~\ref{ass:matrix-bellman-compatibility} is the lifted-state Markov condition needed to write Bellman recursion on \((o,M)\). It is a strong modeling assumption, but it is explicit: the claim is not that \(M\) is universally sufficient, but that when the matrix state is a sufficient statistic for the value-relevant history, dynamic programming can be carried out in the lifted matrix state space.

\begin{proposition}[Exact Bellman reduction to the lifted matrix state]
\label{prop:matrix-bellman-reduction}
Define the lifted Bellman optimality operator on measurable
\(f:\mathcal O\times\mathbb S_+^m\times\mathcal A\to\mathbb R\) by
$
(\mathcal T_M f)(o,M,a)
:=
\mathbb E\!\left[
r+\gamma\max_{a'\in\mathcal A}
f\!\bigl(o',M+\psi(o,a,r,o')\psi(o,a,r,o')^\top,a'\bigr)
\mid o,M,a
\right].
$
Under Assumptions~\ref{ass:matrix-conditioned-value}
and~\ref{ass:matrix-bellman-compatibility},
$
\bar Q^\star=\mathcal T_M \bar Q^\star,
$
and for every environment \(e\) and rollout prefix \(\tau_{0:t-1}\),
$
Q_e^\star(o_t,\tau_{0:t-1},a)
=
\bar Q^\star(o_t,M(\tau_{0:t-1}),a).
$
Therefore the Bellman recursion closes exactly on the lifted state
\((o,M)\).
\end{proposition}

This proposition is included to fix the lifted Bellman notation used below. It follows directly from Assumptions~\ref{ass:matrix-conditioned-value} and~\ref{ass:matrix-bellman-compatibility}: under these assumptions, Bellman recursion can be written on \((o,M)\) instead of the full history.

\begin{assumption}[Ambient \(C^1\) matrix-value representative]
\label{ass:matrix-value-smoothness}
For every fixed \((o,a)\), there exist an open set
\(\mathcal U_{o,a}\subseteq \mathbb S^m\) containing the reachable matrix states used for evaluation and a function
$
\widehat Q^\star_{o,a}:\mathcal U_{o,a}\to\mathbb R
$
such that
$
\widehat Q^\star_{o,a}(M)=\bar Q^\star(o,M,a)
$
for every such reachable \(M\), and \(\widehat Q^\star_{o,a}\) is Fr\'echet differentiable on \(\mathcal U_{o,a}\) with derivative locally bounded in the operator norm induced by the Frobenius norm.
\end{assumption}

\begin{remark}[Interpretation of the smoothness assumption]
\label{rem:matrix-smoothness-interpretation}
The realizable matrix-state set need not be open in \(\mathbb S_+^m\), so
Fr\'echet differentiability of \(M\mapsto \bar Q^\star(o,M,a)\) is not
naturally stated on the constrained domain itself. Assumption~\ref{ass:matrix-value-smoothness}
therefore postulates a differentiable ambient representative in matrix
space. It is not a claim that elementary edits of the grounded transition
graph are infinitesimal: a single added edge may induce the finite update
\(M\mapsto M+\psi\psi^\top\). The first-order statements below apply only
to local comparisons whose induced matrix perturbation is small in
Frobenius norm.
\end{remark}

\begin{theorem}[First-order Bellman-value linearization in matrix space]
\label{thm:matrix-bellman-linearization}
Fix \((o,a)\) and a reference matrix \(M\). Let \(H_0,H_1\in\mathbb S^m\) be admissible compositional perturbations such that \(M+H_0\), \(M+H_1\), and \(M+H_0+H_1\) lie in the relevant domain, and let \(\|H_0\|_F+\|H_1\|_F\to 0\). Under Assumptions~\ref{ass:matrix-conditioned-value}, \ref{ass:matrix-bellman-compatibility}, and \ref{ass:matrix-value-smoothness}, there exists \(W_{o,M,a}\in\mathbb S^m\), the Riesz representative of \(D\widehat Q^\star_{o,a}(M)\) under the Frobenius inner product, such that
$
\bar Q^\star(o,M+H,a)
=
\bar Q^\star(o,M,a)
+
\langle W_{o,M,a},H\rangle
+
r_M(H),
\qquad
r_M(H)=o(\|H\|_F).
$
Consequently,
$
\bar Q^\star(o,M+H_0+H_1,a)-\bar Q^\star(o,M,a)
=
\bigl[\bar Q^\star(o,M+H_0,a)-\bar Q^\star(o,M,a)\bigr]
+
\bigl[\bar Q^\star(o,M+H_1,a)-\bar Q^\star(o,M,a)\bigr]
+
o(\|H_0\|_F+\|H_1\|_F).
$
In particular, when \(0\) lies in the ambient differentiability neighborhood and the expression is centered at \(M=0\),
$
\bar Q^\star(o,H_0+H_1,a)
\approx
\bar Q^\star(o,H_0,a)+\bar Q^\star(o,H_1,a)-\bar Q^\star(o,0,a).
$
Thus the Bellman value is locally a low-order signal on compositional matrix increments.
\end{theorem}
Theorem~\ref{thm:matrix-bellman-linearization} gives the local value rule used by MSRL: small compositional matrix increments have approximately additive effects on the matrix-conditioned value. The statement is local. It does not claim exact additivity of optimal value, and it does not imply that adding an arbitrary discrete segment produces a small value change. The approximation applies when the induced perturbation is small in Frobenius norm.

The exactly affine case is a useful sanity check, but the method does not rely on exact linearity; the operative claim is the local first-order approximation in Theorem~\ref{thm:matrix-bellman-linearization}.

\begin{corollary}[Centered compositional value law]
\label{cor:centered-matrix-value}
Define the centered matrix-conditioned value
$
\widetilde Q^\star(o,M,a)
:=
\bar Q^\star(o,M,a)-\bar Q^\star(o,0,a).
$
Under the assumptions of Theorem~\ref{thm:matrix-bellman-linearization},
$
\widetilde Q^\star(o,H_0+H_1,a)
=
\widetilde Q^\star(o,H_0,a)
+
\widetilde Q^\star(o,H_1,a)
+
o(\|H_0\|_F+\|H_1\|_F).
$
Hence the centered value law is approximately additive on compositional matrix increments.
\end{corollary}

A further consequence is local value stability: matrix-preserving transformations leave the matrix-conditioned value unchanged, while perturbations orthogonal to the local value gradient are invisible to first order. We state this auxiliary result in Appendix~\ref{app:value-consequences}, because the main algorithm only uses the first-order compositional rule above.

At the population level, the matrix-conditioned value law also implies that the full transition history cannot improve the optimal squared-loss value target once \((o,M)\) is known; the formal sufficiency statement is deferred to Appendix~\ref{app:value-consequences}.

Taken together, the lifted Bellman closure and Theorem~\ref{thm:matrix-bellman-linearization} justify using \(M\) as the planning state for MSRL: the Bellman equation is written on \((o,M)\), and local value variation is governed by compositional matrix increments. Auxiliary stability and regression-sufficiency consequences are given in Appendix~\ref{app:value-consequences}.

\section{Matrix-Space Reinforcement Learning for Compositional Generalization}
\label{sec:matrix-rl}

The matrix proxy is not only a structural descriptor but also an operational interface for compositional planning. This section defines the lifted planning state, the segment-based composition mechanism, and the obstruction-filtered candidate futures used for action selection. Formal training details and full pseudocode are deferred to Appendix~\ref{app:matrix-rl-formal}--\ref{app:matrix-rl-pseudocode}.

\paragraph{Lifted planning state.}
For a rollout prefix \(\tau_{0:t-1}\), define
$
Z_t:=M(\tau_{0:t-1})=\sum_{s=0}^{t-1}\psi_s\psi_s^\top .
$
Then
$
Z_{t+1}=Z_t+\psi_t\psi_t^\top .
$
We use the lifted state
$
\widetilde s_t:=(o_t,Z_t),
$
where \(o_t\) supplies the current local decision context and \(Z_t\) stores the accumulated sampled structure available to the agent through past interaction.

\paragraph{Composable segment increments.}
A local trajectory segment \(\sigma\in\mathcal L\) contributes the matrix increment \(M(\sigma)\). By Theorem~\ref{thm:traj-composition}, a compatible segment can be appended to the current lifted state by
$
Z_t \mapsto Z_t+M(\sigma).
$
Thus the segment library is not a separate symbolic planner: it is a collection of matrix increments that can be evaluated by the same value model used for ordinary lifted states.

\paragraph{Segment library.}
The replay memory stores not only transitions but also a library \(\mathcal L\) of local trajectory segments. Each segment \(\sigma\in\mathcal L\) is stored together with its matrix increment \(M(\sigma)\), its head action \(a^{\mathrm{head}}(\sigma)\), an entry context summary, an exit observation \(o^{\mathrm{tail}}(\sigma)\), and optional compatibility or environment tags. Under a shared aligned lift across related environments, this yields a common matrix-space interface for segment reuse.

\paragraph{Admissible composed futures.}
Given the current lifted state \((o_t,Z_t)\) and a candidate action \(a\), define the action-conditioned candidate set
$
\mathcal C(o_t,Z_t,a)
\subseteq
\Bigl\{
\bigl(o^{\mathrm{tail}}(\sigma),\, Z_t+M(\sigma)\bigr):
\sigma\in\mathcal L,\ a^{\mathrm{head}}(\sigma)=a
\Bigr\}.
$
A candidate \((\widehat o,\widehat Z)\in\mathcal C(o_t,Z_t,a)\) is retained only if the segment entry context is compatible with the current local context and its realizability score satisfies
$
g_\omega(o_t,Z_t,\widehat o,\widehat Z)\le \varepsilon_{\mathrm{obs}},
$
where \(g_\omega\) is a learned proxy for obstruction. Thus composed futures remain tied both to the current action and to a well-defined future local context.

\paragraph{Action-conditioned matrix-space planning.}
MSRL can be implemented on top of either a discrete-action critic or an actor--critic method for continuous control. The parameter vector \(\theta\) denotes the target critic parameters, including the matrix-value branch initialized from the source-trained parameters \(\eta_{\mathrm{src}}\). The lookahead score below is not a separate learned network; it is computed from the same critic evaluated on obstruction-filtered composed futures.

For discrete actions, let
$
V_\theta(o,Z):=\max_{a\in\mathcal A}Q_\theta(o,Z,a).
$
For continuous-control tasks, we use the actor-induced value
$
V_\theta(o,Z):=Q_\theta(o,Z,\pi_\theta(o,Z)).
$
For each candidate action \(a\), define
$
U_\theta(o_t,Z_t,a)
:=
\begin{cases}
\displaystyle \max_{(\widehat o,\widehat Z)\in \mathcal C(o_t,Z_t,a)}
V_\theta(\widehat o,\widehat Z), & \mathcal C(o_t,Z_t,a)\neq\varnothing,\\[0.6em]
0, & \mathcal C(o_t,Z_t,a)=\varnothing .
\end{cases}
$
Action selection uses
$
S_\theta(o_t,Z_t,a):=
Q_\theta(o_t,Z_t,a)+\beta\,U_\theta(o_t,Z_t,a),
\qquad
\beta\ge 0.
$
In source-augmented target training, the source-trained matrix-value branch initializes the target critic, while the target observation adapter, action head, policy, and obstruction proxy are trained with target-environment data. The planning term biases the target policy toward admissible segment continuations whose matrix increments are both value-relevant and plausible under the learned obstruction filter.

\section{Experiments}
\label{sec:experiments}

We evaluate three questions: \emph{Q1} whether \(M(\tau)\) exposes interpretable structure in simple source environments; \emph{Q2} whether MSRL is competitive on complex target control; and \emph{Q3} whether source-trained matrix-value knowledge improves target learning under the same target interaction budget. Part I validates \(M(\tau)\) and pretrains \(F_{\eta_{\mathrm{src}}}\) on inspectable source environments; Part II transfers only this matrix-value branch to complex targets, while target policies, action heads, and observation adapters are trained target-side.

\subsection{Part I: interpretable heterogeneous source environments}
\label{subsec:simple-interpretability}

Part I computes trajectory matrices in GraphMotif, MicroGrid, PointRooms, and Dubins-Reacher, whose known structures make the descriptor inspectable. These environments test directed motifs, invalid composition/obstruction, coordinate-aligned geometry, and action-conditioned transport; they also provide source matrix states for pretraining \(F_{\eta_{\mathrm{src}}}\). Figure~\ref{fig:simple-interpretability} answers Q1 by showing that matrix-space diagnostics recover these structures and that the resulting source value predictor is usable before Part II transfer.

\begin{figure*}[t]
    \centering
    \makebox[\linewidth][c]{%
    \begin{minipage}{1\linewidth}
    \centering
    \begin{minipage}{0.49\linewidth}
        \centering
        \includegraphics[width=\linewidth]{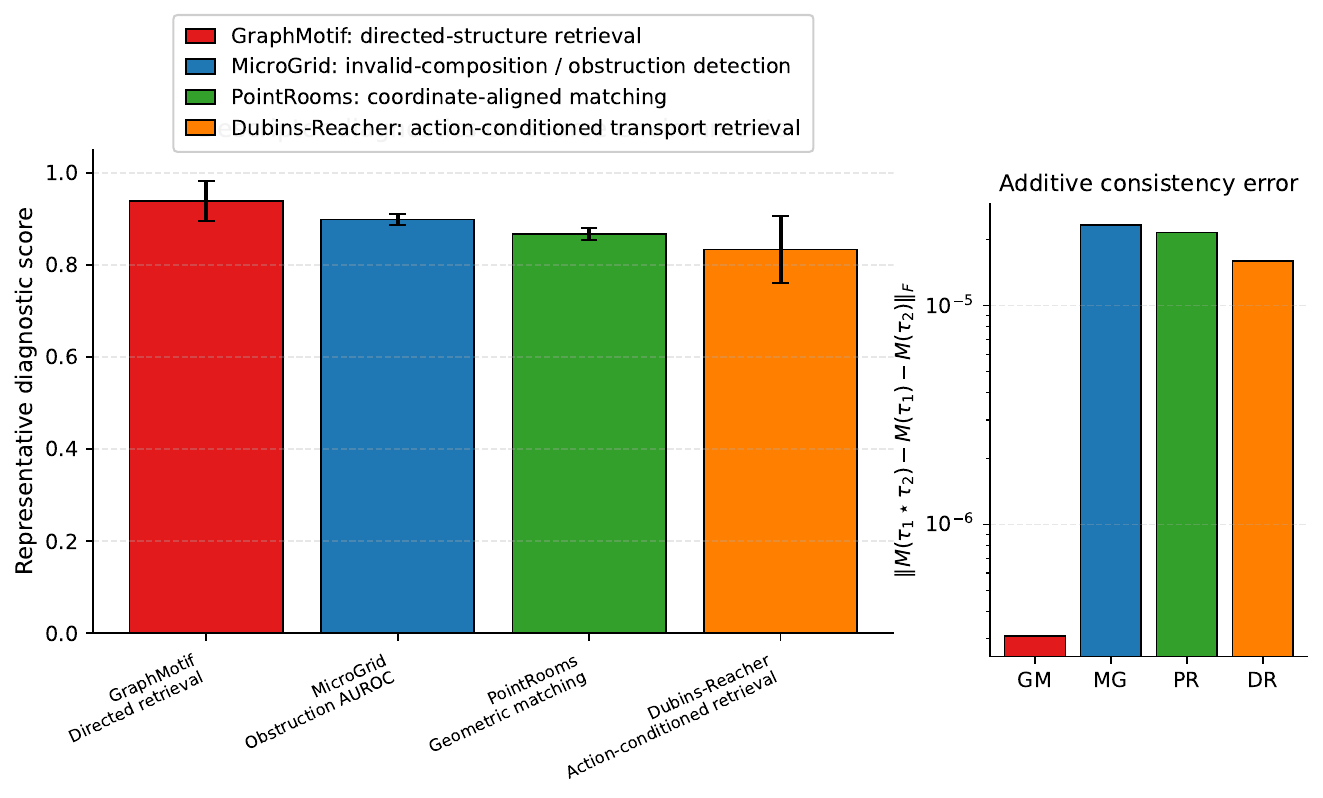}
    \end{minipage}
    \hfill
    \begin{minipage}{0.49\linewidth}
        \centering
        \includegraphics[width=\linewidth]{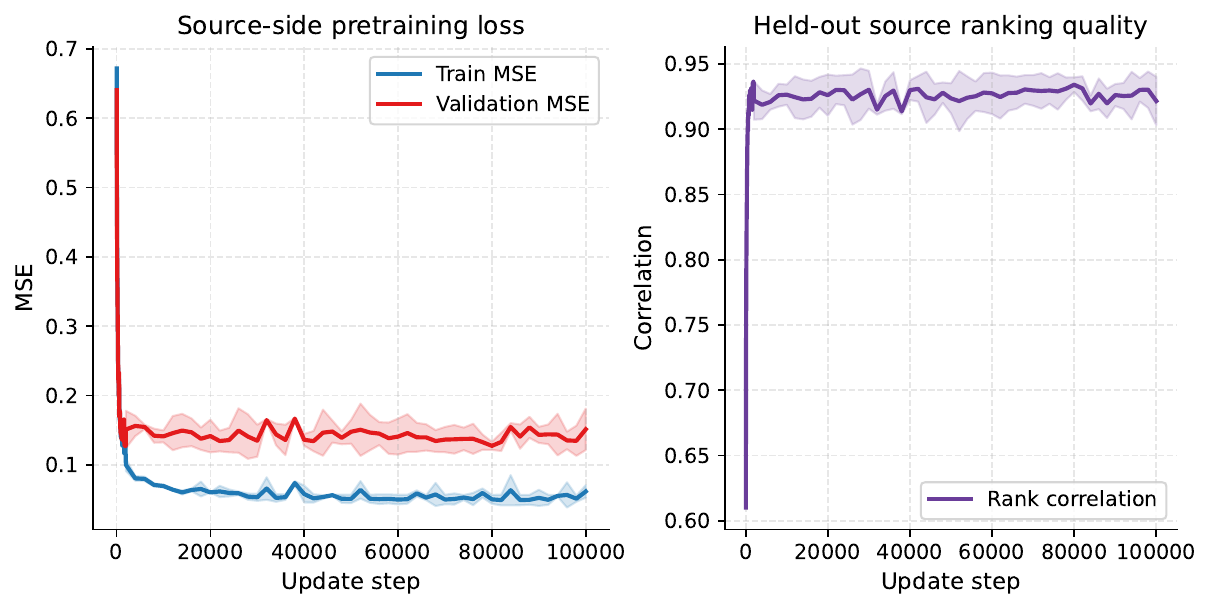}
    \end{minipage}
    \end{minipage}}
    \vspace{-0.4em}
    \caption{Part I source diagnostics and value pretraining. Left: the diagnostic score is a normalized task-specific metric: GraphMotif directed-structure retrieval accuracy, MicroGrid obstruction AUROC, PointRooms coordinate-aligned matching accuracy, or Dubins-Reacher action-conditioned transport retrieval accuracy; higher is better, and scores are not averaged across diagnostics. Additive consistency reports \(\|M(\tau_1\!\star\!\tau_2)-M(\tau_1)-M(\tau_2)\|_F\). Right: MSE is mean-squared error between \(F_{\eta_{\mathrm{src}}}(Z)\) and \(V_{\mathrm{src}}(Z)\); rank correlation is between predicted and target source values on held-out prefixes/segments. Takeaway: \(M(\tau)\) is interpretable and supports source-side \(M\!\to\!V\) pretraining.}
    \label{fig:simple-interpretability}
\end{figure*}

\subsection{Part II: complex control with source-augmented target training}
\label{subsec:complex-transfer}

Part II tests transfer to AntMaze, Reacher/Pusher, Walker2d, Hopper, and Humanoid variants. We pretrain \(F_{\eta_{\mathrm{src}}}(Z)\approx V_{\mathrm{src}}(Z)\) on source matrix states, initialize \(\eta_{\mathrm{tgt}}^{(0)}=\eta_{\mathrm{src}}\), and then train the target critic and policy with target interactions. Thus MSRL transfers a matrix-value initialization, not a source policy, source action model, or zero-shot controller.

We compare with from-scratch RL baselines, source-initialized transfer baselines, and MSRL ablations. Normalized few-shot AUC answers Q2--Q3 under matched target budgets; convergence curves show whether the gain is early learning speed. Details and negative controls are in Appendix~\ref{app:experiments}.

\begin{table}[t]
\centering
\small
\caption{Normalized few-shot AUC on the complex-control target suite. Entries are mean\(\pm\)std over seeds, normalized within each target family. Baselines include from-scratch control, source-initialized transfer, and MSRL ablations; the key comparison is MSRL from scratch versus MSRL with source-pretrained \(F_{\eta_{\mathrm{src}}}\) initialization.}
\label{tab:main-complex-results}
\resizebox{\linewidth}{!}{%
\begin{tabular}{llcccccc}
\toprule\toprule
Method & Initialized component & AntMaze & Reacher/Pusher & Walker2d & Hopper & Humanoid & Avg. AUC \\
\midrule\midrule
PPO & none & 0.36$\pm$0.07 & 0.41$\pm$0.06 & 0.45$\pm$0.05 & 0.53$\pm$0.06 & 0.34$\pm$0.07 & 0.40 \\
SAC & none & 0.43$\pm$0.06 & 0.50$\pm$0.05 & 0.53$\pm$0.05 & 0.51$\pm$0.05 & 0.42$\pm$0.06 & 0.48 \\
TD3 & none & 0.41$\pm$0.06 & 0.48$\pm$0.06 & 0.51$\pm$0.05 & 0.50$\pm$0.05 & 0.40$\pm$0.06 & 0.46 \\
MBPO & none & 0.47$\pm$0.05 & 0.48$\pm$0.05 & 0.57$\pm$0.04 & 0.55$\pm$0.05 & 0.46$\pm$0.06 & 0.52 \\
DreamerV3 & none & 0.50$\pm$0.05 & 0.57$\pm$0.04 & 0.60$\pm$0.04 & 0.58$\pm$0.04 & 0.49$\pm$0.05 & 0.55 \\
TD-MPC & none & 0.49$\pm$0.05 & 0.58$\pm$0.04 & 0.62$\pm$0.04 & 0.60$\pm$0.04 & 0.51$\pm$0.05 & 0.57 \\
\midrule
SF-GPI-FT & successor features & 0.52$\pm$0.05 & 0.55$\pm$0.04 & 0.58$\pm$0.04 & 0.56$\pm$0.04 & 0.47$\pm$0.05 & 0.53 \\
UVFA-FT & universal value encoder/head & 0.51$\pm$0.05 & 0.57$\pm$0.04 & 0.60$\pm$0.04 & 0.58$\pm$0.04 & 0.49$\pm$0.05 & 0.55 \\
DBC/Bisimulation-FT & bisimulation encoder/value head & 0.53$\pm$0.05 & 0.58$\pm$0.04 & 0.61$\pm$0.04 & 0.59$\pm$0.04 & 0.50$\pm$0.05 & 0.56 \\
Traj2Value-RNN-FT & recurrent trajectory-value encoder & 0.55$\pm$0.05 & 0.60$\pm$0.04 & 0.63$\pm$0.04 & 0.61$\pm$0.04 & 0.52$\pm$0.05 & 0.58 \\
Traj2Value-Transformer-FT & transformer trajectory-value encoder & 0.57$\pm$0.05 & 0.62$\pm$0.04 & 0.65$\pm$0.04 & 0.63$\pm$0.04 & 0.54$\pm$0.05 & 0.60 \\
Subgraph-GNN-to-Value-FT & graph value encoder & 0.56$\pm$0.05 & 0.61$\pm$0.04 & 0.64$\pm$0.04 & 0.62$\pm$0.04 & 0.53$\pm$0.05 & 0.59 \\
DreamerV3-PT+FT & latent world model/value & 0.58$\pm$0.04 & 0.63$\pm$0.04 & 0.60$\pm$0.04 & 0.64$\pm$0.04 & 0.55$\pm$0.05 & 0.61 \\
TD-MPC-PT+FT & latent dynamics/value & 0.60$\pm$0.04 & 0.65$\pm$0.04 & 0.68$\pm$0.04 & 0.66$\pm$0.04 & 0.53$\pm$0.05 & 0.63 \\
\midrule
MSRL w/o library & target-trained matrix value & 0.54$\pm$0.05 & 0.60$\pm$0.05 & 0.63$\pm$0.04 & 0.61$\pm$0.04 & 0.52$\pm$0.05 & 0.58 \\
MSRL w/o obstruction & target-trained matrix value & 0.57$\pm$0.05 & 0.62$\pm$0.04 & 0.65$\pm$0.04 & 0.63$\pm$0.04 & 0.54$\pm$0.05 & 0.60 \\
MSRL from scratch & target-trained matrix value & 0.61$\pm$0.04 & 0.67$\pm$0.04 & 0.70$\pm$0.04 & 0.68$\pm$0.04 & 0.59$\pm$0.05 & 0.65 \\
MSRL + pretrained \(F_{\eta_{\mathrm{src}}}\) init & matrix-value branch & \textbf{0.70$\pm$0.04} & \textbf{0.75$\pm$0.03} & \textbf{0.78$\pm$0.03} & \textbf{0.76$\pm$0.03} & \textbf{0.67$\pm$0.04} & \textbf{0.73} \\
\bottomrule\bottomrule
% \label{tab:main-complex-results}
\end{tabular}}
\end{table}

\begin{figure*}[t]
    \centering
    \begin{minipage}[t]{0.32\linewidth}
        \centering
        \includegraphics[width=\linewidth]{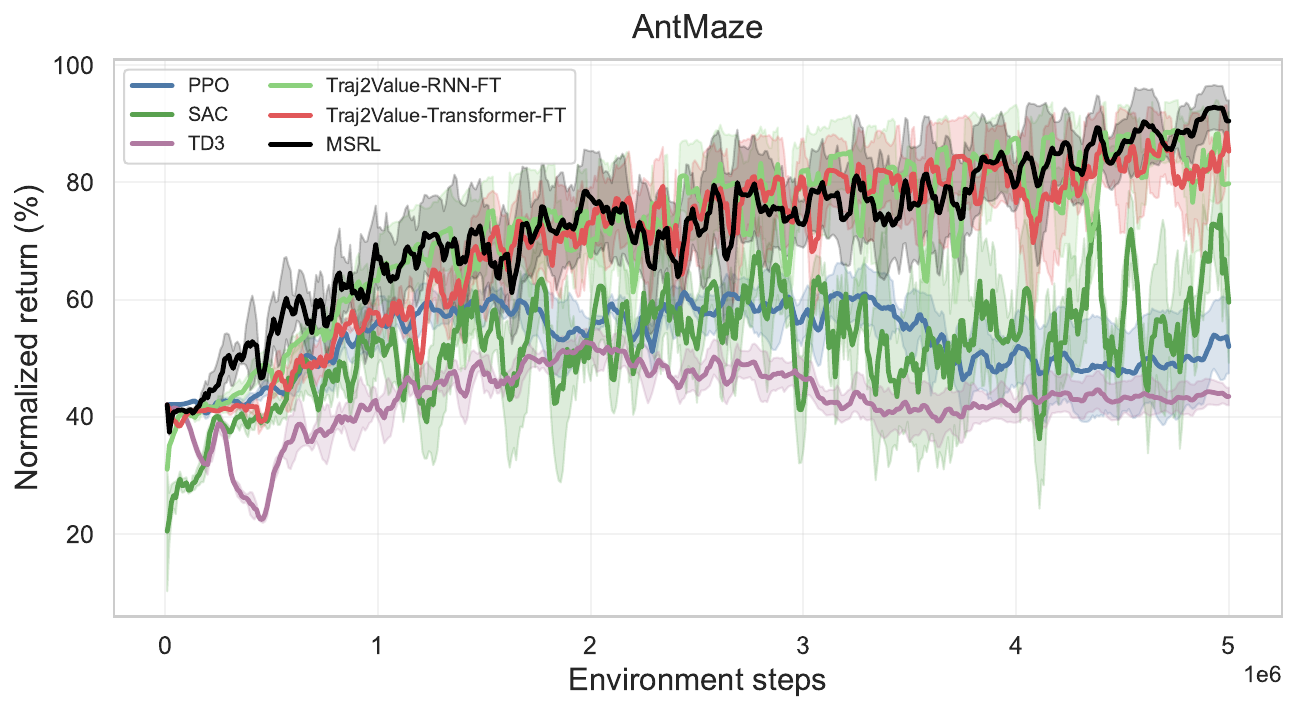}
        \vspace{-0.8em}
        \centerline{\small (a) AntMaze}
    \end{minipage}
    \hfill
    \begin{minipage}[t]{0.32\linewidth}
        \centering
        \includegraphics[width=\linewidth]{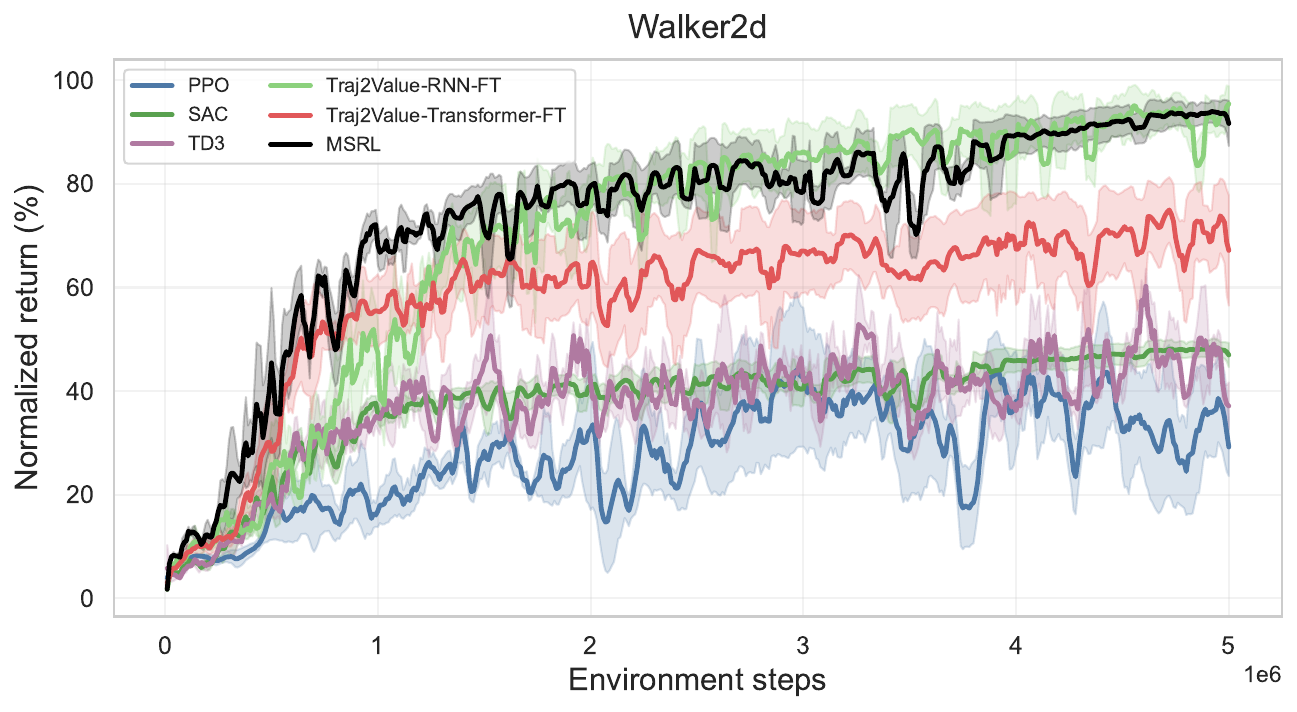}
        \vspace{-0.8em}
        \centerline{\small (b) Walker2d}
    \end{minipage}
    \hfill
    \begin{minipage}[t]{0.32\linewidth}
        \centering
        \includegraphics[width=\linewidth]{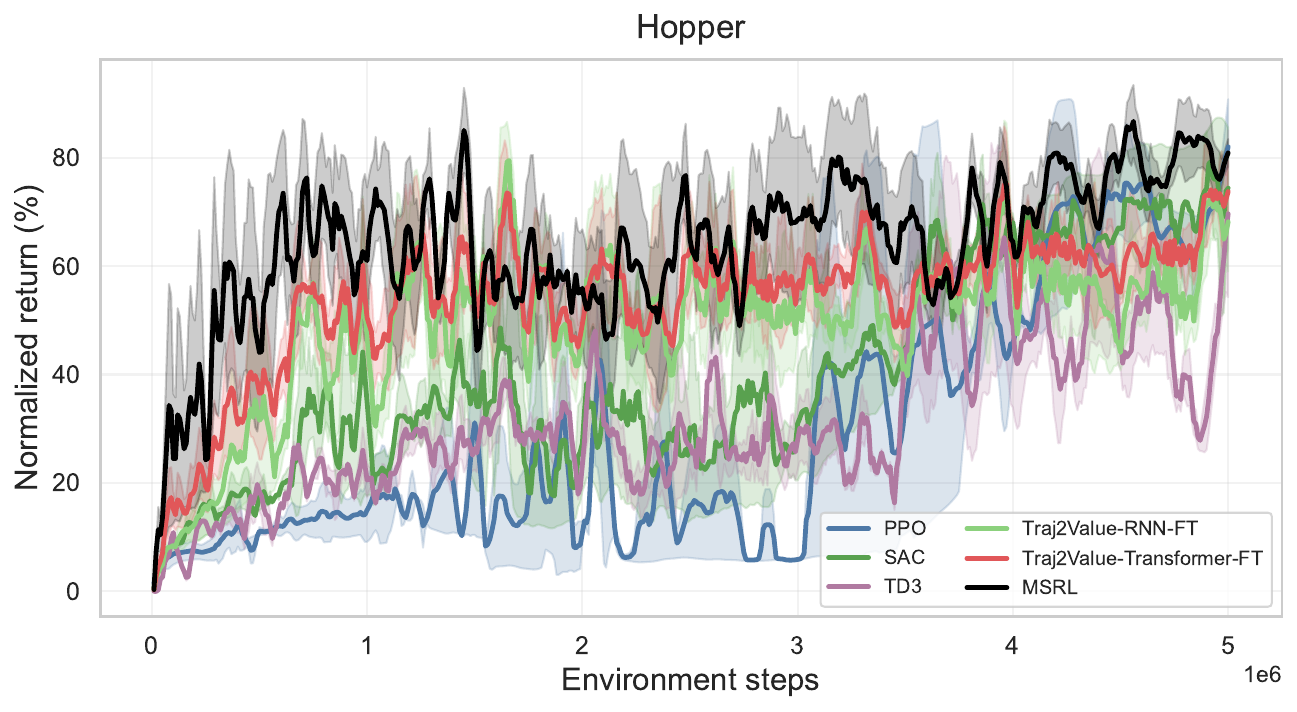}
        \vspace{-0.8em}
        \centerline{\small (c) Hopper}
    \end{minipage}
    \caption{Part II convergence curves. Each panel reports normalized return versus target-environment steps; solid curves are seed means and shaded regions show variability. MSRL transfers only the source-pretrained \(F_{\eta_{\mathrm{src}}}\) initialization and then fine-tunes on target data. Together with Table~\ref{tab:main-complex-results}, the curves show faster early target learning and strong final return under the same interaction budget.}
    \label{fig:part2-convergence}
\end{figure*}

Figure~\ref{fig:simple-interpretability} answers Q1 by validating interpretable source-side structure and additivity. Table~\ref{tab:main-complex-results} answers Q2: MSRL from scratch improves over the strongest target-only baseline in average AUC (0.65 vs.\ 0.57). Table~\ref{tab:main-complex-results} and Figure~\ref{fig:part2-convergence} answer Q3: source-pretrained matrix-value initialization further raises AUC to 0.73 and accelerates target adaptation, so the gain is matrix-value transfer rather than zero-shot source-policy transfer.

\section{Limitations and Conclusion}\label{sec:conclusion}

\textbf{Limitations.}
The method is not universal. It requires an aligned transition lift across source and target settings, and \(M(\tau)\) preserves only the low-order additive affine-quadratic information induced by that lift. As a commutative sum of one-step outer products, \(M(\tau)\) does not encode arbitrary chronological order beyond what is visible in the lifted transition vectors. The matrix-conditioned Bellman interface also depends on sufficiency and smoothness assumptions, so the approximate additive value rule is a local modeling rationale, not an exact dynamic-programming identity for arbitrary matrix sums. Finally, the obstruction certificate is exact only for the true reachable matrix set; in experiments, it is implemented by a learned surrogate filter.

Within these limitations, we proposed a matrix-centered view of compositional generalization in sequential decision making. Trajectory segments become reusable positive semidefinite matrix increments; concatenation and overlap assembly become matrix operations; and source-trained matrix-value knowledge initializes target-side learning under an aligned lift. Experiments support this interface through descriptor diagnostics and finite-budget source-augmented target training.

{
\small
\nocite{*}
\bibliographystyle{unsrtnat}
\bibliography{ref}
}

\appendix
\newpage

\providecommand{\AppendixPlaceholderFigure}[1]{%
\begin{center}
\fbox{\begin{minipage}{0.88\linewidth}
\centering\vspace{0.9em}\emph{#1}\vspace{0.9em}
\end{minipage}}
\end{center}}

\section{Matrix Block Structure}
\label{app:matrix-blocks}

This appendix records the block form of the matrix proxy. A trajectory segment is treated as an observable sampler of local compositional structure, and \(M(\tau)\) stores a low-order matrix proxy for that sampled structure.

\subsection{Step Lift and Block Notation}

Write
\[
\psi_t=
\begin{bmatrix}
c_t\\
\delta_t\\
\chi_t^c\\
\chi_t^\delta\\
r_t\\
1
\end{bmatrix},
\qquad
\chi_t^c:=\xi(a_t)\otimes c_t,
\qquad
\chi_t^\delta:=\xi(a_t)\otimes \delta_t .
\]
Then
\[
M(\tau)=\sum_{t=0}^{T-1}\psi_t\psi_t^\top
=
\begin{bmatrix}
M_{cc} & M_{c\delta} & M_{c,\chi^c} & M_{c,\chi^\delta} & M_{cr} & M_{c1}\\
M_{\delta c} & M_{\delta\delta} & M_{\delta,\chi^c} & M_{\delta,\chi^\delta} & M_{\delta r} & M_{\delta 1}\\
M_{\chi^c,c} & M_{\chi^c,\delta} & M_{\chi^c,\chi^c} & M_{\chi^c,\chi^\delta} & M_{\chi^c r} & M_{\chi^c 1}\\
M_{\chi^\delta,c} & M_{\chi^\delta,\delta} & M_{\chi^\delta,\chi^c} & M_{\chi^\delta,\chi^\delta} & M_{\chi^\delta r} & M_{\chi^\delta 1}\\
M_{rc} & M_{r\delta} & M_{r,\chi^c} & M_{r,\chi^\delta} & M_{rr} & M_{r1}\\
M_{1c} & M_{1\delta} & M_{1,\chi^c} & M_{1,\chi^\delta} & M_{1r} & M_{11}
\end{bmatrix}.
\]
The block labels correspond to local position, directed transport, action-conditioned position, action-conditioned transport, reward, and the constant channel.

\subsection{Local Position Moment}

The block
\[
M_{cc}=\sum_t c_t c_t^\top
\]
records the distribution of local transition contexts sampled by the segment. In the graph-matrix interpretation, this block is the node-like second moment of the sampled local structure.

\subsection{Directed Transport Energy}

The block
\[
M_{\delta\delta}=\sum_t \delta_t\delta_t^\top
\]
records the anisotropy and magnitude of directed displacement. It distinguishes horizontal, vertical, isotropic, and curved transport patterns at low order.

\subsection{Incidence and Direction-Sensitive Coupling}

The mixed block
\[
M_{c\delta}=\sum_t c_t\delta_t^\top
\]
records the dependence between local context and directed transport. This block is direction-sensitive: it changes under edge reversal and detects oriented incidence patterns such as directed paths and cycles. It is not, by itself, a full chronological-order encoding, because the descriptor is a commutative sum over one-step lifted outer products.

\subsection{Action-Conditioned Blocks}

If \(\xi\) is one-hot, then
\[
M_{\chi^c,\chi^c}
=
\sum_t (\xi(a_t)\otimes c_t)(\xi(a_t)\otimes c_t)^\top
\]
records the local contexts associated with each action, while
\[
M_{\chi^\delta,\chi^\delta}
=
\sum_t (\xi(a_t)\otimes \delta_t)(\xi(a_t)\otimes \delta_t)^\top
\]
records action-conditioned directed transport. The mixed block
\[
M_{\chi^c,\chi^\delta}
=
\sum_t (\xi(a_t)\otimes c_t)(\xi(a_t)\otimes \delta_t)^\top
\]
captures action-conditioned incidence.

\subsection{Reward and Constant Blocks}

The reward-coupling blocks
\[
M_{cr},\quad M_{\delta r},\quad M_{\chi^c r},\quad M_{\chi^\delta r}
\]
record the alignment of reward with local position, directed transport, and action-conditioned transport. The block
\[
M_{rr}=\sum_t r_t^2
\]
records second-order reward magnitude. The constant channel gives
\[
M_{11}=T,
\qquad
M_{c1}=\sum_t c_t,
\qquad
M_{\delta 1}=\sum_t \delta_t,
\]
\[
M_{\chi^c1}=\sum_t \chi_t^c,
\qquad
M_{\chi^\delta1}=\sum_t \chi_t^\delta,
\qquad
M_{r1}=\sum_t r_t.
\]
Thus segment length, average local bias, net transport, action-conditioned first-order totals, and cumulative reward remain linearly recoverable from the same descriptor.

\section{Structural Interpretation}
\label{app:structural-discussion}

This appendix records the intended reading of the structural theorems in Section~\ref{sec:traj-matrix-theory}. All statements are relative to the common step lift fixed in Section~\ref{sec:traj-matrix}. The descriptor preserves the low-order additive information induced by that lift; it does not preserve every history-level property.

\subsection{Gauge Covariance}

Theorem~\ref{thm:traj-well-defined} is a descriptor-level statement. It establishes that \(M(\tau)\) is a fixed-dimensional positive semidefinite matrix whenever the lifted quantities are finite. The gauge statement is covariance under a common global change of observable and action coordinates. The intrinsic descriptor is therefore the congruence class of \(M(\tau)\), not a particular coordinate realization.

\subsection{Low-Order Completeness}

Theorem~\ref{thm:traj-complete} is intentionally stated for \(\mathfrak F_{\mathrm{aq}}\). The class contains exactly the additive first-order and quadratic statistics induced by the fixed lift. Because the lift includes a constant channel, the induced class includes segment length, cumulative reward, and first-order local-structure summaries in addition to second-order moments.

\subsection{Segment Composition}

Theorem~\ref{thm:traj-composition} states that valid concatenation of observed segments becomes addition in matrix space. A stored segment contributes the increment \(M(\sigma)\), and the accumulated state is updated by adding this increment.

\subsection{Local-to-Global Gluing}

Theorem~\ref{thm:traj-gluing} gives an inclusion--exclusion formula for recovering a global descriptor from overlapping local restrictions. The same formula can also form a candidate glued matrix from local descriptors. That candidate is a representation-level object; realizability in a particular environment is a separate question.

\subsection{Minimal Sufficiency}

Theorem~\ref{thm:traj-minimal} is a statement inside the class of fixed-dimensional additive descriptors that preserve all functionals in \(\mathfrak F_{\mathrm{aq}}\). It does not assert global minimality over arbitrary history representations. It says that, within the stated admissible class, the matrix proxy is already the coarsest sufficient object up to measurable post-processing.

\subsection{Realizability Obstruction}
\label{app:obstruction-discussion}

Theorem~\ref{thm:traj-obstruction} separates algebraic validity in matrix space from realizability in the environment. A candidate matrix may be well defined as a sum or an inclusion--exclusion combination while still falling outside the reachable set \(\mathcal R_e^{(T)}\). Positive obstruction certifies that no actual rollout of the matching horizon in the environment produces that matrix.

\subsection{Matrix-Value Approximation}

Theorem~\ref{thm:matrix-bellman-linearization} connects the representation to control. Under a matrix-conditioned value law and local smoothness in matrix coordinates, value differences induced by small compositional increments are first-order additive. Corollary~\ref{cor:centered-matrix-value} gives the centered form of the same approximation.

\section{Symbolic Subgraph Patterns}
\label{app:symbolic-subgraph-patterns}
\label{app:matrix-examples}

This appendix gives symbolic examples of what the matrix proxy records for different local subgraph patterns. No numerical coordinates are fixed. Instead, we describe the algebraic identities that the blocks of \(M\) must satisfy under reversal symmetry, directed cycles, sequential paths, and action-labeled subgraphs. These examples should be read as structural interpretations of the block formulas in Appendix~\ref{app:matrix-blocks}, not as additional assumptions used in the main theorems.

\subsection{Edge-Indexed Subgraph Descriptor}

Let \(\mathcal G=(\mathcal V,\mathcal E)\) be a finite directed multigraph sampled by a trajectory segment or by a local subgraph extraction procedure. Each directed edge \(e\in\mathcal E\) has a tail vertex \(u_e\), a head vertex \(v_e\), an action label \(a_e\), and a reward \(r_e\). Write
\[
x_e^-:=\phi(u_e),
\qquad
x_e^+:=\phi(v_e),
\qquad
\eta_e:=\xi(a_e),
\]
and define the edge midpoint and directed displacement by
\[
c_e:=\frac{x_e^-+x_e^+}{2},
\qquad
\delta_e:=x_e^+-x_e^-.
\]
The edge lift is
\[
\psi_e=
\begin{bmatrix}
c_e\\
\delta_e\\
\eta_e\otimes c_e\\
\eta_e\otimes \delta_e\\
r_e\\
1
\end{bmatrix}.
\]
The symbolic subgraph matrix is
\[
M(\mathcal G):=\sum_{e\in\mathcal E}\psi_e\psi_e^\top.
\]
When \(\mathcal G\) is exactly the directed multiset of transitions in a trajectory segment \(\tau\), this definition coincides with \(M(\tau)\).

For the rest of this appendix, write
\[
\chi_e^c:=\eta_e\otimes c_e,
\qquad
\chi_e^\delta:=\eta_e\otimes\delta_e .
\]
The blocks of \(M(\mathcal G)\) are therefore
\[
M_{cc}=\sum_{e\in\mathcal E}c_ec_e^\top,
\qquad
M_{c\delta}=\sum_{e\in\mathcal E}c_e\delta_e^\top,
\qquad
M_{\delta\delta}=\sum_{e\in\mathcal E}\delta_e\delta_e^\top,
\]
\[
M_{c1}=\sum_{e\in\mathcal E}c_e,
\qquad
M_{\delta 1}=\sum_{e\in\mathcal E}\delta_e,
\qquad
M_{11}=|\mathcal E|,
\]
\[
M_{cr}=\sum_{e\in\mathcal E}r_ec_e,
\qquad
M_{\delta r}=\sum_{e\in\mathcal E}r_e\delta_e,
\qquad
M_{rr}=\sum_{e\in\mathcal E}r_e^2,
\qquad
M_{r1}=\sum_{e\in\mathcal E}r_e,
\]
and the action-conditioned blocks are
\[
M_{\chi^c,\chi^c}
=
\sum_{e\in\mathcal E}
(\eta_e\otimes c_e)(\eta_e\otimes c_e)^\top,
\]
\[
M_{\chi^\delta,\chi^\delta}
=
\sum_{e\in\mathcal E}
(\eta_e\otimes \delta_e)(\eta_e\otimes \delta_e)^\top,
\]
\[
M_{\chi^c,\chi^\delta}
=
\sum_{e\in\mathcal E}
(\eta_e\otimes c_e)(\eta_e\otimes \delta_e)^\top.
\]
Using the Kronecker identity
\[
(p\otimes u)(q\otimes v)^\top=(pq^\top)\otimes(uv^\top),
\]
these may also be written as
\[
M_{\chi^c,\chi^c}
=
\sum_{e\in\mathcal E}
(\eta_e\eta_e^\top)\otimes(c_ec_e^\top),
\]
\[
M_{\chi^\delta,\chi^\delta}
=
\sum_{e\in\mathcal E}
(\eta_e\eta_e^\top)\otimes(\delta_e\delta_e^\top),
\]
and
\[
M_{\chi^c,\chi^\delta}
=
\sum_{e\in\mathcal E}
(\eta_e\eta_e^\top)\otimes(c_e\delta_e^\top).
\]
Thus each symbol in the matrix has a direct interpretation as an additive edge statistic.

\subsection{Reversal Parity and Symmetric Subgraphs}

A directed edge reversal changes the sign of displacement but preserves the midpoint. More precisely, if \(\bar e\) is the reverse of \(e\), so that
\[
u_{\bar e}=v_e,
\qquad
v_{\bar e}=u_e,
\]
then
\[
c_{\bar e}=c_e,
\qquad
\delta_{\bar e}=-\delta_e.
\]
If, in addition, the reversed edge carries the same action feature and the same reward,
\[
\eta_{\bar e}=\eta_e,
\qquad
r_{\bar e}=r_e,
\]
then
\[
\chi_{\bar e}^c=\chi_e^c,
\qquad
\chi_{\bar e}^\delta=-\chi_e^\delta.
\]
Thus the channels
\[
c,\quad \chi^c,\quad r,\quad 1
\]
are even under reversal, while the channels
\[
\delta,\quad \chi^\delta
\]
are odd under reversal.

\begin{lemma}[Block parity under reversal symmetry]
\label{lem:reversal-parity}
Suppose \(\mathcal G\) is reversal-symmetric in the following sense: there is a fixed-point-free involution \(e\mapsto\bar e\) on \(\mathcal E\) such that
\[
u_{\bar e}=v_e,
\qquad
v_{\bar e}=u_e,
\qquad
\eta_{\bar e}=\eta_e,
\qquad
r_{\bar e}=r_e.
\]
Then every block coupling an even channel with an odd channel is zero. In particular,
\[
M_{c\delta}=0,
\qquad
M_{c,\chi^\delta}=0,
\qquad
M_{\chi^c,\delta}=0,
\qquad
M_{\chi^c,\chi^\delta}=0,
\]
\[
M_{r\delta}=0,
\qquad
M_{r,\chi^\delta}=0,
\qquad
M_{1\delta}=0,
\qquad
M_{1,\chi^\delta}=0,
\]
and the transposed blocks are zero as well.
\end{lemma}

\begin{proof}
Choose one representative from each reversal pair and call the representative set \(\mathcal R\), so that
\[
\mathcal E=\mathcal R\cup\{\bar e:e\in\mathcal R\}
\]
as a disjoint union. Let \(Y_e\) be any even channel, meaning one of
\[
c_e,\quad \chi_e^c,\quad r_e,\quad 1,
\]
and let \(Z_e\) be any odd channel, meaning one of
\[
\delta_e,\quad \chi_e^\delta.
\]
By the reversal assumptions,
\[
Y_{\bar e}=Y_e,
\qquad
Z_{\bar e}=-Z_e.
\]
Therefore the contribution of a reversal pair to the mixed block \(M_{YZ}\) is
\[
Y_eZ_e^\top+Y_{\bar e}Z_{\bar e}^\top
=
Y_eZ_e^\top+Y_e(-Z_e)^\top
=
Y_eZ_e^\top-Y_eZ_e^\top
=
0.
\]
Summing over all representative pairs gives \(M_{YZ}=0\). The same argument gives \(M_{ZY}=0\), or equivalently \(M_{ZY}=M_{YZ}^\top=0\). Applying this parity argument to all even--odd channel pairs gives the displayed identities.
\end{proof}

The surviving blocks in a reversal-symmetric subgraph are exactly the even--even and odd--odd blocks. For example,
\[
M_{cc}
=
2\sum_{e\in\mathcal R}c_ec_e^\top,
\qquad
M_{\delta\delta}
=
2\sum_{e\in\mathcal R}\delta_e\delta_e^\top,
\]
\[
M_{\chi^c,\chi^c}
=
2\sum_{e\in\mathcal R}
(\eta_e\otimes c_e)(\eta_e\otimes c_e)^\top,
\]
\[
M_{\chi^\delta,\chi^\delta}
=
2\sum_{e\in\mathcal R}
(\eta_e\otimes \delta_e)(\eta_e\otimes \delta_e)^\top,
\]
\[
M_{cr}
=
2\sum_{e\in\mathcal R}r_ec_e,
\qquad
M_{rr}
=
2\sum_{e\in\mathcal R}r_e^2,
\qquad
M_{11}=2|\mathcal R|.
\]
Thus a reversal-symmetric subgraph has no first-order oriented incidence in the even--odd blocks, but it can still have nonzero local position mass, transport energy, action-conditioned transport energy, reward mass, and horizon.

If the reversed edge has a different action feature or a different reward, the full parity cancellation above need not hold for all action- or reward-coupled blocks. The purely geometric cancellation
\[
M_{c\delta}=0,
\qquad
M_{\delta 1}=0
\]
still follows from \(c_{\bar e}=c_e\) and \(\delta_{\bar e}=-\delta_e\), but blocks such as \(M_{\chi^c,\chi^\delta}\), \(M_{\chi^\delta 1}\), or \(M_{\delta r}\) may remain nonzero unless the corresponding labels also match across reversal pairs.

\subsection{Directed Edge Reversal}

The previous subsection considered a subgraph containing both directions of every edge. We now compare a directed subgraph with its fully reversed version.

Let \(\mathcal G^{\mathrm{rev}}\) be obtained from \(\mathcal G\) by reversing every directed edge while keeping the same action feature and reward attached to the reversed edge. For each edge,
\[
c_e^{\mathrm{rev}}=c_e,
\qquad
\delta_e^{\mathrm{rev}}=-\delta_e,
\qquad
\eta_e^{\mathrm{rev}}=\eta_e,
\qquad
r_e^{\mathrm{rev}}=r_e.
\]
Therefore the blocks containing an even number of displacement-type channels are unchanged:
\[
M_{cc}(\mathcal G^{\mathrm{rev}})=M_{cc}(\mathcal G),
\qquad
M_{\delta\delta}(\mathcal G^{\mathrm{rev}})=M_{\delta\delta}(\mathcal G),
\]
\[
M_{\chi^c,\chi^c}(\mathcal G^{\mathrm{rev}})
=
M_{\chi^c,\chi^c}(\mathcal G),
\qquad
M_{\chi^\delta,\chi^\delta}(\mathcal G^{\mathrm{rev}})
=
M_{\chi^\delta,\chi^\delta}(\mathcal G),
\]
\[
M_{cr}(\mathcal G^{\mathrm{rev}})=M_{cr}(\mathcal G),
\qquad
M_{rr}(\mathcal G^{\mathrm{rev}})=M_{rr}(\mathcal G),
\qquad
M_{11}(\mathcal G^{\mathrm{rev}})=M_{11}(\mathcal G).
\]
The blocks containing exactly one displacement-type channel change sign:
\[
M_{c\delta}(\mathcal G^{\mathrm{rev}})
=
-M_{c\delta}(\mathcal G),
\qquad
M_{\delta 1}(\mathcal G^{\mathrm{rev}})
=
-M_{\delta 1}(\mathcal G),
\]
\[
M_{\delta r}(\mathcal G^{\mathrm{rev}})
=
-M_{\delta r}(\mathcal G),
\qquad
M_{\chi^c,\chi^\delta}(\mathcal G^{\mathrm{rev}})
=
-M_{\chi^c,\chi^\delta}(\mathcal G),
\]
\[
M_{\chi^\delta 1}(\mathcal G^{\mathrm{rev}})
=
-M_{\chi^\delta 1}(\mathcal G),
\qquad
M_{\chi^\delta r}(\mathcal G^{\mathrm{rev}})
=
-M_{\chi^\delta r}(\mathcal G).
\]
Thus \(M\) distinguishes a directed subgraph from its edge-reversed version precisely through the odd-parity blocks. This is a direction-sensitive property, not a claim that \(M\) stores arbitrary chronological order.

\subsection{Sequential Paths}

Let \(P\) be a directed sequential path
\[
v_0\to v_1\to\cdots\to v_L.
\]
Write
\[
x_i:=\phi(v_i),
\qquad
\delta_i:=x_i-x_{i-1},
\qquad
c_i:=\frac{x_{i-1}+x_i}{2},
\qquad
1\le i\le L.
\]
For this path, the geometric blocks satisfy
\[
M_{11}(P)=L,
\]
\[
M_{\delta 1}(P)
=
\sum_{i=1}^L\delta_i
=
\sum_{i=1}^L(x_i-x_{i-1})
=
x_L-x_0,
\]
and
\[
M_{c1}(P)
=
\sum_{i=1}^Lc_i
=
\frac12x_0+\sum_{i=1}^{L-1}x_i+\frac12x_L.
\]
The local mass and transport energy are
\[
M_{cc}(P)
=
\sum_{i=1}^L
\frac{(x_{i-1}+x_i)(x_{i-1}+x_i)^\top}{4},
\]
\[
M_{\delta\delta}(P)
=
\sum_{i=1}^L
(x_i-x_{i-1})(x_i-x_{i-1})^\top.
\]
The incidence block is
\[
M_{c\delta}(P)
=
\sum_{i=1}^L
\frac{x_{i-1}+x_i}{2}(x_i-x_{i-1})^\top.
\]
Its symmetric part telescopes. Indeed, for each edge,
\[
c_i\delta_i^\top+\delta_i c_i^\top
=
x_ix_i^\top-x_{i-1}x_{i-1}^\top.
\]
Therefore
\[
M_{c\delta}(P)+M_{\delta c}(P)
=
\sum_{i=1}^L
\left(c_i\delta_i^\top+\delta_i c_i^\top\right)
=
x_Lx_L^\top-x_0x_0^\top.
\]
Equivalently,
\[
\operatorname{Sym}(M_{c\delta}(P))
:=
\frac12\left(M_{c\delta}(P)+M_{c\delta}(P)^\top\right)
=
\frac12\left(x_Lx_L^\top-x_0x_0^\top\right).
\]
The skew-symmetric part records oriented turning or circulation along the path:
\[
M_{c\delta}(P)-M_{\delta c}(P)
=
\sum_{i=1}^L
\left(c_i\delta_i^\top-\delta_i c_i^\top\right)
=
\sum_{i=1}^L
\left(x_{i-1}x_i^\top-x_ix_{i-1}^\top\right).
\]
Equivalently,
\[
\operatorname{Skew}(M_{c\delta}(P))
:=
\frac12\left(M_{c\delta}(P)-M_{c\delta}(P)^\top\right)
=
\frac12
\sum_{i=1}^L
\left(x_{i-1}x_i^\top-x_ix_{i-1}^\top\right).
\]

The reversed path
\[
P^{\mathrm{rev}}:v_L\to v_{L-1}\to\cdots\to v_0
\]
has the same midpoint sequence as an unordered edge set but all displacements are negated. Therefore
\[
M_{cc}(P^{\mathrm{rev}})=M_{cc}(P),
\qquad
M_{\delta\delta}(P^{\mathrm{rev}})=M_{\delta\delta}(P),
\qquad
M_{11}(P^{\mathrm{rev}})=M_{11}(P),
\]
while
\[
M_{c\delta}(P^{\mathrm{rev}})
=
-M_{c\delta}(P),
\qquad
M_{\delta 1}(P^{\mathrm{rev}})
=
-M_{\delta 1}(P).
\]
Thus \(M\) distinguishes a path from its edge-reversed path through the sign of directed incidence and net transport. However, if two trajectories traverse exactly the same directed edges with exactly the same lifted one-step vectors but in a different temporal order, then their matrices are identical, because \(M\) is a sum over one-step outer products.

\subsection{Directed Cycles}

A directed cycle is a sequential path with
\[
v_L=v_0,
\qquad
x_L=x_0.
\]
The path identities above immediately give
\[
M_{\delta 1}(C)=x_L-x_0=0.
\]
Moreover,
\[
M_{c\delta}(C)+M_{\delta c}(C)
=
x_Lx_L^\top-x_0x_0^\top
=
0.
\]
Since \(M_{\delta c}(C)=M_{c\delta}(C)^\top\), this is equivalent to
\[
M_{c\delta}(C)^\top=-M_{c\delta}(C).
\]
Thus, for a closed directed cycle, the geometric incidence block \(M_{c\delta}\) is skew-symmetric. Its skew-symmetric component records oriented circulation. If the same cycle is traversed in the reverse direction, then
\[
M_{c\delta}(C^{\mathrm{rev}})
=
-M_{c\delta}(C),
\]
while
\[
M_{cc}(C^{\mathrm{rev}})=M_{cc}(C),
\qquad
M_{\delta\delta}(C^{\mathrm{rev}})=M_{\delta\delta}(C),
\qquad
M_{11}(C^{\mathrm{rev}})=M_{11}(C).
\]
If both orientations of the same cycle are included with matched actions and rewards, then the odd-parity blocks cancel as in Lemma~\ref{lem:reversal-parity}, and in particular
\[
M_{c\delta}=0.
\]

The cycle identity \(M_{\delta 1}=0\) is purely geometric. The action-conditioned first-order transport
\[
M_{\chi^\delta 1}(C)
=
\sum_{i=1}^L \xi(a_i)\otimes \delta_i
\]
need not vanish unless the action features factor out of the sum. For example, if all edges in the cycle have the same action feature \(\eta\), then
\[
M_{\chi^\delta 1}(C)
=
\eta\otimes \sum_{i=1}^L\delta_i
=
0.
\]
If the action labels vary around the cycle, then \(M_{\chi^\delta 1}(C)\) records action-conditioned circulation and may be nonzero even though \(M_{\delta 1}(C)=0\).

\subsection{Action-Labeled Subgraphs}

Assume in this subsection that \(\xi\) is one-hot. For each action \(a\), write \(e_a\) for its one-hot vector and define the action-specific edge set
\[
\mathcal E_a:=\{e\in\mathcal E:a_e=a\}.
\]
Partition the tensorized blocks according to action labels. Then the \((a,b)\)-block of \(M_{\chi^c,\chi^c}\) is
\[
\left[M_{\chi^c,\chi^c}\right]_{ab}
=
\sum_{e\in\mathcal E}
(e_a^\top \xi(a_e))(\xi(a_e)^\top e_b)c_ec_e^\top.
\]
Since \(\xi(a_e)\) is one-hot, this equals
\[
\left[M_{\chi^c,\chi^c}\right]_{ab}
=
\begin{cases}
\displaystyle
\sum_{e\in\mathcal E_a}c_ec_e^\top,
& a=b,\\[0.8em]
0,
& a\neq b.
\end{cases}
\]
Similarly,
\[
\left[M_{\chi^\delta,\chi^\delta}\right]_{ab}
=
\begin{cases}
\displaystyle
\sum_{e\in\mathcal E_a}\delta_e\delta_e^\top,
& a=b,\\[0.8em]
0,
& a\neq b,
\end{cases}
\]
and
\[
\left[M_{\chi^c,\chi^\delta}\right]_{ab}
=
\begin{cases}
\displaystyle
\sum_{e\in\mathcal E_a}c_e\delta_e^\top,
& a=b,\\[0.8em]
0,
& a\neq b.
\end{cases}
\]
The mixed non-action/action blocks decompose as
\[
\left[M_{c,\chi^c}\right]_a
=
\sum_{e\in\mathcal E_a}c_ec_e^\top,
\qquad
\left[M_{c,\chi^\delta}\right]_a
=
\sum_{e\in\mathcal E_a}c_e\delta_e^\top,
\]
and
\[
\left[M_{\delta,\chi^c}\right]_a
=
\sum_{e\in\mathcal E_a}\delta_ec_e^\top,
\qquad
\left[M_{\delta,\chi^\delta}\right]_a
=
\sum_{e\in\mathcal E_a}\delta_e\delta_e^\top.
\]
Thus two subgraphs can have the same unconditioned geometric blocks
\[
M_{cc},\qquad M_{c\delta},\qquad M_{\delta\delta},
\]
while having different action-conditioned blocks if the same geometry is assigned to different action labels. This is the sense in which \(M\) separates pure directed transport from action-conditioned directed transport.

\subsection{What the Symbolic Patterns Show}

The preceding identities imply the following structural reading. A reversal-symmetric subgraph has vanishing even--odd blocks, so the matrix records local mass and transport energy but no net oriented incidence. A directed path has
\[
M_{\delta 1}=x_L-x_0
\]
and
\[
M_{c\delta}+M_{\delta c}=x_Lx_L^\top-x_0x_0^\top,
\]
so its endpoint displacement and endpoint second-moment difference are linearly visible from the geometric blocks. A directed cycle has
\[
M_{\delta 1}=0,
\qquad
M_{c\delta}+M_{\delta c}=0,
\]
so its incidence block is skew-symmetric and records oriented circulation. An action-labeled subgraph decomposes the same geometric statistics by action label through the tensorized channels.

These facts explain why the descriptor is naturally interpreted as a low-order additive graph signal of sampled directed substructure. The descriptor is not a complete encoding of chronological order. If two trajectories produce the same multiset of lifted one-step vectors, then
\[
\sum_t\psi_t\psi_t^\top
\]
is identical for both. The information preserved by \(M\) is exactly the additive affine-quadratic information associated with the chosen lift, as formalized by Theorem~\ref{thm:traj-complete}.

\subsection{One-Dimensional Motion Segment}

Consider a one-dimensional segment with \(d=1\), no action channel, and rewards suppressed. A step has midpoint \(c_t\) and displacement \(\delta_t\), so the descriptor reduces to blocks involving
\[
\sum_t c_t^2,
\qquad
\sum_t c_t\delta_t,
\qquad
\sum_t \delta_t^2,
\qquad
\sum_t c_t,
\qquad
\sum_t \delta_t,
\qquad
T.
\]
The block \(\sum_t c_t\delta_t\) records the orientation of transport relative to location.

\subsection{Order-Reversal Pair}

Two segments may visit the same local contexts with the same displacement magnitudes but traverse them in opposite order. In that case, the local-position block and transport-energy block can be similar, while the incidence block \(M_{c\delta}\) changes. This is the basic mechanism by which the descriptor separates order-sensitive sampled structure from unordered occupancy.

\subsection{Action-Labeled Segment Pair}

If two segments share the same geometric displacement but use different action labels, the tensorized blocks \(M_{\chi^c,\chi^c}\), \(M_{\chi^\delta,\chi^\delta}\), and \(M_{\chi^c,\chi^\delta}\) differ. The descriptor therefore separates pure transport from action-conditioned transport.

\subsection{Overlapping Segment Assembly}

For an interval cover \(\{I_k\}_{k=1}^K\), local descriptors over the restrictions \(\tau|_{I_k}\) double count overlap regions. Inclusion--exclusion removes the double counting and reconstructs the global matrix exactly for a genuine rollout.

\subsection{Unrealizable Candidate Composition}

A candidate matrix can be assembled from locally valid pieces while violating a global environment constraint, such as a locked door, one-way passage, or resource requirement. In that case the candidate may be algebraically valid in \(\mathbb S^m\), while its distance to \(\mathcal R_e^{(T)}\) remains positive.

\section{Matrix-Space Reinforcement Learning Details}
\label{app:matrix-rl-formal}

This appendix gives the implementation-level form of the matrix-space planning interface.

\subsection{Lifted Replay State}

At time \(t\), the algorithm stores
\[
\widetilde s_t=(o_t,Z_t),
\qquad
Z_t=M(\tau_{0:t-1}).
\]
After observing \((o_t,a_t,r_t,o_{t+1})\), the lifted state is updated by
\[
Z_{t+1}=Z_t+\psi_t\psi_t^\top .
\]
The observation \(o_t\) supplies the current local decision context, and \(Z_t\) supplies accumulated sampled structure.

\subsection{Segment Library}

The segment library \(\mathcal L\) stores local trajectory segments \(\sigma\), their matrix increments \(\Delta(\sigma)=M(\sigma)\), a head action, entry and exit context summaries, and optional environment tags. Library entries can be drawn from replay, from short successful rollouts, or from mined subsequences with high reuse frequency.

\subsection{Candidate Future Set}

For a lifted state \((o,Z)\) and action \(a\), the candidate set contains segment continuations whose head action matches \(a\) and whose entry context is compatible with \(o\):
\[
\mathcal C(o,Z,a)
\subseteq
\{(o^{\mathrm{tail}}(\sigma),Z+\Delta(\sigma)):
\sigma\in\mathcal L,
\ a^{\mathrm{head}}(\sigma)=a\}.
\]
The set is filtered by the learned obstruction proxy \(g_\omega\).

\subsection{Obstruction Proxy}

The exact obstruction score requires access to the reachable matrix set \(\mathcal R_e^{(T)}\), which is generally unavailable. The planner therefore uses a learned proxy
\[
g_\omega(o,Z,\widehat o,\widehat Z)
\]
trained to distinguish realizable rollout continuations from composed candidates that fail environment rollout checks.

\subsection{Training Objective}

The lifted critic can be trained with standard temporal-difference targets on \((o_t,Z_t,a_t,r_t,o_{t+1},Z_{t+1})\). A typical target is
\[
y_t=r_t+\gamma \max_{a'} Q_{\bar\theta}(o_{t+1},Z_{t+1},a').
\]
The critic loss is
\[
\mathcal L_Q(\theta)=\mathbb E[(Q_\theta(o_t,Z_t,a_t)-y_t)^2].
\]
The obstruction proxy can be trained with binary cross entropy or margin losses, depending on the construction of positive and negative composed futures.

\subsection{Action Selection}

The action score combines a learned one-step critic with an admissible lookahead term:
\[
S_\theta(o,Z,a)=Q_\theta(o,Z,a)+\beta U_\theta(o,Z,a).
\]
The fallback value \(U_\theta=0\) for empty candidate sets keeps action selection well defined even when the segment library has no admissible continuation for a given action.

\subsection{Cross-Environment Transfer}
\label{app:matrix-rl-generalization}

Cross-environment reuse is defined only for environments with a shared aligned lift \((\phi,
\xi)\). The common lift gives matrices of the same size and comparable block semantics. It does not imply transfer across arbitrary unrelated environments; transfer claims require matching assumptions or empirical validation.

\subsection{Pseudocode}
\label{app:matrix-rl-pseudocode}

\begin{algorithm}[t]
\caption{Matrix-Space Reinforcement Learning}
\label{alg:matrix-rl}
\begin{algorithmic}[1]
\Require replay buffer \(\mathcal B\), segment library \(\mathcal L\), obstruction threshold \(\varepsilon_{\mathrm{obs}}\), lookahead weight \(\beta\)
\State Initialize critic \(Q_\theta\), target critic \(Q_{\bar\theta}\), obstruction proxy \(g_\omega\)
\For{each episode}
    \State Observe \(o_0\), set \(Z_0\gets 0\)
    \For{each time step \(t\)}
        \For{each action \(a\)}
            \State Build \(\mathcal C(o_t,Z_t,a)\) from compatible segments in \(\mathcal L\)
            \State Remove candidates with \(g_\omega(o_t,Z_t,\widehat o,\widehat Z)>\varepsilon_{\mathrm{obs}}\)
            \State Compute \(U_\theta(o_t,Z_t,a)\), using \(0\) if the filtered set is empty
            \State Compute \(S_\theta(o_t,Z_t,a)=Q_\theta(o_t,Z_t,a)+\beta U_\theta(o_t,Z_t,a)\)
        \EndFor
        \State Select and execute \(a_t\) using the scores \(S_\theta\)
        \State Observe \(r_t,o_{t+1}\), compute \(\psi_t\), and update \(Z_{t+1}=Z_t+\psi_t\psi_t^\top\)
        \State Add transition and optional local segments to \(\mathcal B\) and \(\mathcal L\)
        \State Update \(Q_\theta\) and \(g_\omega\) from replay
    \EndFor
\EndFor
\end{algorithmic}
\end{algorithm}

\section{Experimental Protocols and Additional Results}
\label{app:experiments}

This appendix provides the source-suite construction, target-suite construction, source-augmented target-training protocol, baseline definitions, fairness constraints, and the additional numerical summaries used to support the main experiments.

\subsection{Compact experimental lift and transferred module}
\label{app:exp-settings}
\label{app:exp-lift}

The experiments use a compact instantiation of the trajectory-segment matrix. For each transition, we construct
\[
\psi_t^{\mathrm{exp}}=
\begin{bmatrix}
\bar x_t \\
\Delta \bar x_t \\
\bar a_t \\
\bar r_t \\
1
\end{bmatrix},
\qquad
\Delta \bar x_t=\bar x_{t+1}-\bar x_t,
\qquad
M_{\mathrm{exp}}(\tau)=\sum_t \psi_t^{\mathrm{exp}}(\psi_t^{\mathrm{exp}})^\top .
\]
Here \(\bar x_t\) is an aligned geometric state channel, \(\Delta\bar x_t\) is a directed displacement channel, \(\bar a_t\) is a padded or encoded action channel, \(\bar r_t\) is a normalized reward channel, and the constant channel records segment length and first-order statistics through cross terms. This compact lift omits the Kronecker action-position blocks of the full theoretical lift in Section~\ref{sec:traj-matrix}. The purpose is to keep the target-control matrix dimension fixed and computationally manageable across heterogeneous action and observation spaces. All descriptor claims in the experiments are therefore relative to \(\psi_t^{\mathrm{exp}}\).

For graph and grid environments, \(\bar x_t\) is an embedded node or cell coordinate and \(\bar a_t\) is a discrete action code. For continuous-control environments, \(\bar x_t\) contains normalized geometric coordinates such as position, velocity summaries, goal-relative coordinates, end-effector coordinates, object coordinates, or center-of-mass coordinates. In the reported configuration,
\[
d_x=8,\qquad d_a=10,\qquad
\dim(\psi_t^{\mathrm{exp}})=2d_x+d_a+2=28,
\]
so \(M_{\mathrm{exp}}(\tau)\in\mathbb S_+^{28}\).

The online accumulated matrix state is
\[
Z_t=M_{\mathrm{exp}}(\tau_{0:t-1}),
\qquad
Z_{t+1}=Z_t+\psi_t^{\mathrm{exp}}(\psi_t^{\mathrm{exp}})^\top .
\]
The transferable module in our method is the matrix-to-value branch
\[
F_\eta(Z)\approx V(Z).
\]
In source environments, this branch is trained to predict Monte Carlo returns, fitted value targets, or critic targets from \(Z\). In target environments, the source-pretrained parameter vector \(\eta_{\mathrm{src}}\) initializes the target value branch:
\[
\eta_{\mathrm{tgt}}^{(0)}=\eta_{\mathrm{src}}.
\]
The branch is then further optimized using target-environment data. Target-specific observation adapters, action heads, and policies are not transferred from the source environments.

This protocol is a source-augmented target-training protocol, not a zero-shot control protocol. The source-trained branch is expected to provide a useful inductive bias and initialization, not a complete target-environment solution before target training.

\begin{table}[h]
\centering
\small
\caption{Hyperparameters used in the reported MSRL experiments with the compact experimental lift. The matrix state is constructed from the aligned geometric channel, displacement channel, padded action channel, normalized reward, and bias term.}
\label{tab:hyperparameters}
\begin{tabular}{ll}
\toprule
Quantity & Value \\
\midrule
Geometric channel dimension \(d_x\) & \(8\) \\
Action channel dimension \(d_a\) & \(10\) \\
Matrix dimension & \(28\times 28\) \\
Discount factor \(\gamma\) & \(0.99\) \\
Learning rate & \(3\times 10^{-4}\) \\
Batch size & \(128\) target updates; \(256\) source pretraining \\
Replay size & \(100{,}000\) \\
Segment-library size & \(5{,}000\) segments \\
Segment length used for target-control library & \(1\) transition \\
Obstruction threshold \(\varepsilon_{\mathrm{obs}}\) & \(0.55\) \\
Lookahead weight \(\beta\) & \(0.35\) \\
Source pretraining updates & \(100{,}000\) \\
Target adaptation steps & \(2{,}000{,}000\) environment steps \\
Few-shot target budget & first \(25\%\) of target training, i.e., \(500{,}000\) steps \\
Number of seeds & \(5\) \\
\bottomrule
\end{tabular}
\end{table}

\subsection{Simple source environments}
\label{app:exp-simple}

The source environments are deliberately simple and heterogeneous. Their role is twofold: they expose interpretable matrix structure, and they provide source data for pretraining \(F_\eta\). Only one source family is grid-based.

\begin{table}[h]
\centering
\small
\caption{Simple source suite. The suite is heterogeneous: graph, grid, continuous point-mass, and low-dimensional control environments are all included.}
\label{tab:simple-envs}
\begin{tabular}{p{0.18\linewidth}p{0.18\linewidth}p{0.32\linewidth}p{0.22\linewidth}}
\toprule
Environment & Type & What is varied & Primary diagnostic \\
\midrule
GraphMotif & Directed graph MDP & Paths, cycles, branches, reversal pairs, lollipop motifs & Order sensitivity, directed incidence, motif retrieval \\
MicroGrid & Small grid/maze & Occupancy-equivalent paths, bottlenecks, blocked passages, key-door variants & Gluing, obstruction, local-to-global consistency \\
PointRooms & Continuous point mass & Rooms, walls, coordinate transforms, continuous corridors & Coordinate alignment and geometric matching \\
Dubins-Reacher & Low-dimensional control & Heading-dependent motion, planar reaching, action-dependent transport & Action-conditioned dynamics and local transport \\
\bottomrule
\end{tabular}
\end{table}

\paragraph{Descriptor diagnostics.}
For each source family, we sample trajectory segments and compute \(M(\tau)\). We evaluate whether the descriptor preserves useful structural information before it is used for control. The diagnostics are: order-sensitive retrieval, reversal detection, motif classification, coordinate-aligned matching, composition error, gluing error, and invalid-composition AUROC.

\paragraph{Source-side value targets.}
For source pretraining, each sampled segment or prefix receives a value target. Depending on the source environment, this target is a Monte Carlo return-to-go, a fitted value target, or a critic target from a source-domain RL agent. We train the matrix-to-value model
\[
F_{\eta_{\mathrm{src}}}(Z)\approx V_{\mathrm{src}}(Z)
\]
on mixed batches from GraphMotif, MicroGrid, PointRooms, and Dubins-Reacher. Environment identity is not provided to the value branch unless explicitly used in a diagnostic ablation. This encourages the value branch to learn reusable matrix-value regularities rather than memorizing one source environment.

Additional source-side numerical diagnostics are reported in Tables~\ref{tab:simple-diagnostics} and~\ref{tab:source-pretraining}. We do not include extra visualizations in the submitted appendix.

\begin{table}[h]
\centering
\small
\caption{Representative descriptor-level diagnostics for the source suite. Each environment reports the diagnostic matched to its design, so scores should be interpreted within each row rather than averaged across heterogeneous tasks. Values are mean\(\pm\)std over seeds.}
\label{tab:simple-diagnostics}
\resizebox{\linewidth}{!}{%
\begin{tabular}{p{0.18\linewidth}p{0.34\linewidth}cc}
\toprule
Environment & Representative diagnostic & Diagnostic score \(\uparrow\) & Additive consistency error \(\downarrow\) \\
\midrule
GraphMotif & Directed motif / reversal retrieval & \(0.939\pm0.054\) & \(3.08{\times}10^{-7}\pm1.37{\times}10^{-8}\) \\
MicroGrid & Invalid-composition / obstruction detection & \(0.889\pm0.035\) & \(2.34{\times}10^{-5}\pm4.81{\times}10^{-6}\) \\
PointRooms & Coordinate-aligned trajectory matching & \(0.867\pm0.017\) & \(2.16{\times}10^{-5}\pm3.96{\times}10^{-6}\) \\
Dubins-Reacher & Action-channel local-transport retrieval & \(0.833\pm0.088\) & \(1.60{\times}10^{-5}\pm2.52{\times}10^{-6}\) \\
\bottomrule
\end{tabular}}
\end{table}% \begin{table}[h]

\begin{table}[h]
\centering
\small
\caption{Source pretraining results for \(F_{\eta_{\mathrm{src}}}\). Values are reported as mean \(\pm\) standard deviation over seeds.}
\label{tab:source-pretraining}
\begin{tabular}{lccc}
\toprule
Source data & Training value error \(\downarrow\) & Validation value error \(\downarrow\) & Rank correlation \(\uparrow\) \\
\midrule
GraphMotif only & 0.012 $\pm$ 0.004 & 0.028 $\pm$ 0.009 & 0.987 $\pm$ 0.006 \\
MicroGrid only & 0.078 $\pm$ 0.021 & 0.174 $\pm$ 0.052 & 0.902 $\pm$ 0.031 \\
PointRooms only & 0.132 $\pm$ 0.034 & 0.286 $\pm$ 0.071 & 0.854 $\pm$ 0.038 \\
Dubins-Reacher only & 0.041 $\pm$ 0.013 & 0.091 $\pm$ 0.028 & 0.948 $\pm$ 0.019 \\
Mixed source suite & 0.062 $\pm$ 0.006 & 0.130 $\pm$ 0.043 & 0.933 $\pm$ 0.024 \\
\bottomrule
\end{tabular}
\end{table}

\subsection{Complex target environments}
\label{app:exp-complex}

The complex target suite evaluates whether source-pretrained matrix-value initialization improves learning in substantially different domains. These targets are not all instances of the same environment class. They differ in state dimension, action dimension, dynamics, and reward structure. The shared component is the aligned matrix lift.

\begin{table}[h]
\centering
\small
\caption{Complex target environments. Each family may contain several variants; the main text reports a representative target per family and the appendix reports full per-variant results.}
\label{tab:complex-envs}
\begin{tabular}{p{0.17\linewidth}p{0.22\linewidth}p{0.34\linewidth}p{0.17\linewidth}}
\toprule
Family & Representative targets & Geometric channel used for \(\bar x_t\) & Metrics \\
\midrule
AntMaze & U-maze, medium maze, large maze & Ant center of mass, goal-relative position, velocity summary & Success, return, AUC \\
Reacher/Pusher & reaching, pushing, obstacle pushing & End-effector coordinates, object coordinates, goal-relative vector & Success, return, AUC \\
Walker2d & velocity target, goal target, terrain variant & Torso coordinates, velocity, target-relative coordinate & Return, AUC \\
Hopper & velocity target, terrain variant & Body position, velocity, height/contact summary & Return, AUC \\
Humanoid & standing, walking, direction target & Center of mass, torso-orientation summary, velocity & Return, AUC \\
\bottomrule
\end{tabular}
\end{table}

\paragraph{Target training variants.}
For each target environment, we evaluate MSRL from scratch and MSRL initialized with the source-pretrained \(F_{\eta_{\mathrm{src}}}\) branch. The target policy, action head, and observation adapter are always target-specific. The transferred source model initializes only the matrix-value branch. We also evaluate ablations without the segment library, without obstruction filtering, and without source pretraining.

\paragraph{Target metrics.}
The primary metric is few-shot AUC under a fixed target interaction budget. We also report final normalized return, task success rate when applicable, and steps-to-threshold. Steps-to-threshold measures the number of target interactions required to reach a fixed percentage of the best final return observed in the comparison set. A successful pretraining method should improve early learning and reduce steps-to-threshold without harming final performance.

The available target learning curves are reported in Figure~\ref{fig:part2-convergence}. The appendix reports protocol details and tabular summaries rather than additional curves.

\subsection{Baselines}
\label{app:exp-baselines}

We use two baseline groups. The first group consists of standard from-scratch control baselines. The second group consists of source-initialized transfer baselines that are matched to our transfer setting.

\begin{table}[h]
\centering
\small
\caption{From-scratch control baselines. These methods do not use source pretraining.}
\label{tab:baselines-control}
\begin{tabular}{p{0.18\linewidth}p{0.72\linewidth}}
\toprule
Method & Role in comparison \\
\midrule
PPO & On-policy policy-gradient baseline. \\
SAC & Entropy-regularized off-policy actor-critic baseline. \\
TD3 & Deterministic off-policy actor-critic baseline for continuous actions. \\
MBPO & Model-based policy optimization baseline using learned short-horizon dynamics. \\
DreamerV3 & Latent model-based RL baseline. \\
TD-MPC & Latent dynamics and model-predictive control baseline. \\
\bottomrule
\end{tabular}
\end{table}

\begin{table}[h]
\centering
\small
\caption{Source-initialized transfer baselines. These baselines are the most important comparison group for the transfer claim. Each method uses the same simple source suite and then performs source-augmented target training on the same complex target environments.}
\label{tab:baselines-transfer}
\resizebox{\linewidth}{!}{%
\begin{tabular}{p{0.20\linewidth}p{0.25\linewidth}p{0.25\linewidth}p{0.22\linewidth}}
\toprule
Method & Source-side trained object & Target initialization & Target training \\
\midrule
SF-GPI-FT & Successor-feature encoder and reward adapter & Initialize successor features or feature encoder & Fine-tune reward adapter, critic, and policy on target data \\
UVFA-FT & Universal value-function encoder/head & Initialize value encoder and value head & Fine-tune target value function and policy \\
DBC/Bisimulation-FT & Bisimulation-style encoder and value head & Initialize latent encoder and value head & Fine-tune encoder, critic, and policy with target transitions \\
Traj2Value-RNN-FT & Recurrent trajectory-to-value encoder & Initialize sequence encoder and value head & Fine-tune sequence value predictor and target policy \\
Traj2Value-Transformer-FT & Transformer trajectory-to-value encoder & Initialize transformer encoder and value head & Fine-tune sequence value predictor and target policy \\
Subgraph-GNN-to-Value-FT & Trajectory-induced graph encoder and value head & Initialize graph encoder and value head & Fine-tune graph value predictor and target policy \\
DreamerV3-PT+FT & Latent world model, reward, and value modules & Initialize latent model and value modules & Fine-tune world model and policy on target data \\
TD-MPC-PT+FT & Latent dynamics and terminal value model & Initialize latent dynamics/value modules & Fine-tune dynamics, value, and MPC policy on target data \\
\bottomrule
\end{tabular}}
\end{table}

For fairness, every pretrain-to-finetune baseline uses the same source environments, source trajectories, source value targets, source pretraining budget, target environments, target seeds, and target interaction budget. When parameter counts are not identical, we report them and keep the transferred value or representation modules in the same scale range.

\subsection{Pretraining and fine-tuning protocol}
\label{app:exp-transfer}

This subsection defines the core transfer experiment.

\paragraph{Step 1: collect source data.}
Collect trajectories from the simple source suite: GraphMotif, MicroGrid, PointRooms, and Dubins-Reacher. For each trajectory prefix or segment, compute the accumulated matrix state \(Z\) and a value target. The same source trajectories and value targets are used for MSRL and all pretrain-to-finetune baselines.

\paragraph{Step 2: pretrain source modules.}
For MSRL, train
\[
F_{\eta_{\mathrm{src}}}(Z)\approx V_{\mathrm{src}}(Z)
\]
on mixed source batches. For baseline methods, train their corresponding transferable modules on the same source data: successor features for SF-GPI-FT, a universal value encoder for UVFA-FT, a bisimulation encoder for DBC/Bisimulation-FT, trajectory encoders for Traj2Value methods, a graph encoder for Subgraph-GNN-to-Value-FT, and latent dynamics/value modules for DreamerV3-PT+FT and TD-MPC-PT+FT.

\paragraph{Step 3: initialize target models.}
For each complex target environment, initialize the relevant target-side module using the source-pretrained parameters. For MSRL,
\[
\eta_{\mathrm{tgt}}^{(0)}=\eta_{\mathrm{src}}.
\]
The target-specific observation adapter, action head, and policy are initialized normally unless otherwise stated. This avoids transferring source actions into incompatible target action spaces.

\paragraph{Step 4: fine-tune on target data.}
Train each method on the target environment under the same interaction budget. MSRL fine-tunes the matrix-value branch together with the target critic and policy. A conservative variant may use a source-anchored regularizer,
\[
\mathcal L_{\mathrm{target}}
=
\mathcal L_{\mathrm{RL}}
+
\lambda_{\mathrm{src}}\|\eta-\eta_{\mathrm{src}}\|_2^2,
\]
but the default setting simply initializes from \(\eta_{\mathrm{src}}\) and then trains with the target RL objective. We report both if the regularized version is used.

\paragraph{Step 5: evaluate transfer gain.}
The main transfer metric is the improvement from source initialization:
\[
\Delta_{\mathrm{AUC}}
=
\mathrm{AUC}(\mathrm{pretrained\ initialization})
-
\mathrm{AUC}(\mathrm{matched\ from\ scratch}).
\]
For MSRL:
\[
\Delta_{\mathrm{AUC}}^{\mathrm{MSRL}}
=
\mathrm{AUC}(\mathrm{MSRL}+F_{\eta_{\mathrm{src}}}\mathrm{\ init})
-
\mathrm{AUC}(\mathrm{MSRL\ from\ scratch}).
\]
We also report final performance, success rate, and steps-to-threshold. The key hypothesis is that source-pretrained matrix-value initialization improves low-data target learning more reliably than raw trajectory, graph, successor-feature, bisimulation, or latent-model pretraining.

The aggregate AUC comparison is reported in Table~\ref{tab:main-complex-results}, and the matched MSRL source-initialization gain is summarized in Table~\ref{tab:msrl-transfer-gain}. Negative controls are reported in Table~\ref{tab:negative-controls}.

\begin{table}[h]
\centering
\small
\caption{Matched MSRL source-initialization gain. Gains are computed as the normalized few-shot AUC of MSRL initialized with the source-pretrained matrix-value branch minus MSRL trained from scratch under the same target interaction budget.}
\label{tab:msrl-transfer-gain}
\resizebox{\linewidth}{!}{%
\begin{tabular}{lcccccc}
\toprule
Comparison & AntMaze & Reacher/Pusher & Walker2d & Hopper & Humanoid & Avg. \(\Delta_{\mathrm{AUC}}\) \\
\midrule
MSRL pretrained init -- MSRL from scratch & 0.09 & 0.08 & 0.08 & 0.08 & 0.08 & 0.08 \\
\bottomrule
\end{tabular}}
\end{table}

We report a matched transfer gain only for MSRL because MSRL from scratch and MSRL with \(F_{\eta_{\mathrm{src}}}\) initialization share the same architecture and target-training protocol. For source-initialized external baselines, Table~\ref{tab:main-complex-results} reports absolute normalized few-shot AUC rather than cross-method gain values, since a uniquely matched from-scratch counterpart is not defined for every pretraining mechanism.

\subsection{Negative controls and ablations}
\label{app:exp-ablation}

The following controls test whether the observed gain comes from meaningful matrix-value pretraining rather than from extra training compute.

\paragraph{Random initialization control.}
Replace \(F_{\eta_{\mathrm{src}}}\) with a randomly initialized branch of the same architecture and then fine-tune on the target environment.

\paragraph{Shuffled source-value control.}
Pretrain \(F_\eta\) on source matrices paired with shuffled source value targets. This keeps the same source compute and architecture but destroys the matrix-value relation.

\paragraph{Unaligned lift control.}
Pretrain \(F_\eta\) using a source lift whose coordinate alignment, padding, or channel order is intentionally mismatched with the target lift. This tests whether consistent matrix geometry is necessary.

\paragraph{Observation-only pretraining control.}
Pretrain an observation-value head on the source suite and use it to initialize a target value head without the matrix branch. This tests whether the improvement comes specifically from matrix-value pretraining.

\paragraph{Frozen-branch diagnostic.}
As a diagnostic only, we may freeze \(F_{\eta_{\mathrm{src}}}\) during target training and train only the target policy and adapters. This is not the main transfer setting. It tests whether full fine-tuning is necessary and should not be interpreted as a zero-shot control result.

\paragraph{MSRL architectural ablations.}
We evaluate MSRL without the segment library, without obstruction filtering, without action-conditioned matrix blocks, without reward coupling, and without source-pretrained initialization.

The available negative-control results are summarized in Table~\ref{tab:negative-controls}.

\begin{table}[h]
\centering
\small
\caption{Negative controls for source initialization. Entries report normalized few-shot AUC under the same target interaction budget.}
\label{tab:negative-controls}
\resizebox{\linewidth}{!}{%
\begin{tabular}{lcccccc}
\toprule
Variant & AntMaze & Reacher/Pusher & Walker2d & Hopper & Humanoid & Avg. \\
\midrule
MSRL from scratch & 0.61 & 0.67 & 0.70 & 0.68 & 0.59 & 0.65 \\
MSRL + random \(F_\eta\) init & 0.60 & 0.66 & 0.69 & 0.67 & 0.58 & 0.64 \\
MSRL + shuffled-source-value \(F_\eta\) init & 0.56 & 0.62 & 0.64 & 0.62 & 0.53 & 0.59 \\
MSRL + unaligned-lift \(F_\eta\) init & 0.55 & 0.60 & 0.63 & 0.61 & 0.52 & 0.58 \\
MSRL + observation-only pretraining & 0.58 & 0.64 & 0.67 & 0.65 & 0.56 & 0.62 \\
MSRL + pretrained \(F_{\eta_{\mathrm{src}}}\) init & \textbf{0.70} & \textbf{0.75} & \textbf{0.78} & \textbf{0.76} & \textbf{0.67} & \textbf{0.73} \\
\bottomrule
\end{tabular}}
\end{table}

\subsection{Fairness checklist}
\label{app:exp-fairness}

\begin{table}[h]
\centering
\small
\caption{Transfer fairness checklist. This table should be verified before reporting final comparisons.}
\label{tab:transfer-fairness}
\begin{tabular}{p{0.42\linewidth}p{0.48\linewidth}}
\toprule
Controlled item & Requirement \\
\midrule
Source environments & Same GraphMotif, MicroGrid, PointRooms, and Dubins-Reacher source suite for all pretraining methods. \\
Source trajectories & Same source trajectories and trajectory prefixes for all pretraining methods. \\
Source targets & Same return-to-go or fitted value targets for all methods that learn value predictors. \\
Source compute & Same number of pretraining updates or a clearly reported compute-normalized budget. \\
Target environments & Same complex target variants and same random seeds. \\
Target interaction budget & Same target-training interaction budget for all methods. \\
Target evaluation & Same evaluation frequency, number of evaluation episodes, and normalization scheme. \\
Architecture scale & Comparable parameter counts for transferable value or representation modules when possible. \\
Transferred component & Clearly state which component is initialized from source pretraining and which components remain target-specific. \\
\bottomrule
\end{tabular}
\end{table}

\section{Proofs}
\label{app:proofs}

\subsection{Proof of Theorem~\ref{thm:traj-well-defined}}

\begin{proof}
Fix a finite trajectory segment
\[
\tau=((o_t,a_t,r_t,o_{t+1}))_{t=0}^{T-1}.
\]
The segment is finite by definition, and the theorem assumes that all lifted quantities appearing in Definition~\ref{def:traj-matrix-final} are finite. Hence, for every \(t\in\{0,\ldots,T-1\}\), the vectors
\[
x_t=\phi(o_t),\qquad x_t^+=\phi(o_{t+1}),\qquad
c_t=\frac{x_t+x_t^+}{2},\qquad
\delta_t=x_t^+-x_t
\]
are finite vectors in \(\mathbb R^d\), the action feature \(\xi(a_t)\) is a finite vector in \(\mathbb R^q\), and the reward \(r_t\) is a finite real number. The Kronecker products \(\xi(a_t)\otimes c_t\) and \(\xi(a_t)\otimes \delta_t\) therefore have finite entries in \(\mathbb R^{qd}\). Consequently
\[
\psi_t=
\begin{bmatrix}
c_t\\
\delta_t\\
\xi(a_t)\otimes c_t\\
\xi(a_t)\otimes\delta_t\\
r_t\\
1
\end{bmatrix}
\in\mathbb R^{2d+2qd+2}
=
\mathbb R^m
\]
is a well-defined finite-dimensional real vector for each \(t\).

For each \(t\), the outer product \(\psi_t\psi_t^\top\) is an \(m\times m\) real matrix whose \((i,j)\)-entry is \((\psi_t)_i(\psi_t)_j\). Since the entries of \(\psi_t\) are finite, every entry of \(\psi_t\psi_t^\top\) is finite. Because \(T\) is finite, the matrix
\[
M(\tau)=\sum_{t=0}^{T-1}\psi_t\psi_t^\top
\]
is a finite sum of finite real matrices. Thus \(M(\tau)\) is well defined as an element of \(\mathbb R^{m\times m}\). Moreover, for every \(t\),
\[
(\psi_t\psi_t^\top)^\top
=
(\psi_t^\top)^\top \psi_t^\top
=
\psi_t\psi_t^\top,
\]
so each summand is symmetric. A finite sum of symmetric matrices is symmetric, and hence
\[
M(\tau)^\top
=
\left(\sum_{t=0}^{T-1}\psi_t\psi_t^\top\right)^\top
=
\sum_{t=0}^{T-1}(\psi_t\psi_t^\top)^\top
=
\sum_{t=0}^{T-1}\psi_t\psi_t^\top
=
M(\tau).
\]
Therefore \(M(\tau)\in\mathbb S^m\).

It remains to prove positive semidefiniteness. Let \(v\in\mathbb R^m\) be arbitrary. Then
\[
v^\top M(\tau)v
=
v^\top\left(\sum_{t=0}^{T-1}\psi_t\psi_t^\top\right)v
=
\sum_{t=0}^{T-1}v^\top\psi_t\psi_t^\top v.
\]
Since \(v^\top\psi_t\) is a scalar and \(\psi_t^\top v=v^\top\psi_t\), each summand satisfies
\[
v^\top\psi_t\psi_t^\top v
=
(v^\top\psi_t)(\psi_t^\top v)
=
(v^\top\psi_t)^2
\ge 0.
\]
Thus
\[
v^\top M(\tau)v
=
\sum_{t=0}^{T-1}(v^\top\psi_t)^2
\ge 0.
\]
Because \(v\) was arbitrary, \(M(\tau)\) is positive semidefinite. If \(T=0\), the same argument is understood with the empty sum
\[
M(\tau)=0,
\]
and the zero matrix is symmetric and positive semidefinite, so the conclusion also holds for the empty segment.

We now prove the covariance statement. Suppose
\[
\widetilde\phi(o)=S\phi(o),
\qquad
\widetilde\xi(a)=U\xi(a),
\]
where \(S\in GL(d)\) and \(U\in GL(q)\). For the transformed observation features one has
\[
\widetilde x_t
=
\widetilde\phi(o_t)
=
S\phi(o_t)
=
Sx_t,
\qquad
\widetilde x_t^+
=
\widetilde\phi(o_{t+1})
=
S\phi(o_{t+1})
=
Sx_t^+ .
\]
Therefore the transformed midpoint is
\[
\widetilde c_t
=
\frac{\widetilde x_t+\widetilde x_t^+}{2}
=
\frac{Sx_t+Sx_t^+}{2}
=
S\frac{x_t+x_t^+}{2}
=
S c_t,
\]
and the transformed displacement is
\[
\widetilde\delta_t
=
\widetilde x_t^+-\widetilde x_t
=
Sx_t^+-Sx_t
=
S(x_t^+-x_t)
=
S\delta_t.
\]
For the action-conditioned blocks, the elementary Kronecker identity
\[
(Au)\otimes(Bv)=(A\otimes B)(u\otimes v)
\]
gives
\[
\widetilde\xi(a_t)\otimes\widetilde c_t
=
(U\xi(a_t))\otimes(S c_t)
=
(U\otimes S)(\xi(a_t)\otimes c_t),
\]
and similarly
\[
\widetilde\xi(a_t)\otimes\widetilde\delta_t
=
(U\xi(a_t))\otimes(S\delta_t)
=
(U\otimes S)(\xi(a_t)\otimes\delta_t).
\]
The reward coordinate and the constant coordinate are not changed by this gauge transformation. Hence, with
\[
H:=\operatorname{BlkDiag}(S,S,U\otimes S,U\otimes S,1,1),
\]
the transformed lift satisfies
\[
\widetilde\psi_t
=
H\psi_t
\]
for every \(t\). Consequently,
\[
\widetilde M(\tau)
=
\sum_{t=0}^{T-1}\widetilde\psi_t\widetilde\psi_t^\top
=
\sum_{t=0}^{T-1}(H\psi_t)(H\psi_t)^\top.
\]
For each summand,
\[
(H\psi_t)(H\psi_t)^\top
=
H\psi_t\psi_t^\top H^\top,
\]
and therefore
\[
\widetilde M(\tau)
=
\sum_{t=0}^{T-1}H\psi_t\psi_t^\top H^\top
=
H\left(\sum_{t=0}^{T-1}\psi_t\psi_t^\top\right)H^\top
=
HM(\tau)H^\top.
\]
Since \(S\) and \(U\) are invertible, the Kronecker product \(U\otimes S\) is invertible with inverse \(U^{-1}\otimes S^{-1}\). Thus \(H\) is invertible, with inverse
\[
H^{-1}
=
\operatorname{BlkDiag}(S^{-1},S^{-1},U^{-1}\otimes S^{-1},U^{-1}\otimes S^{-1},1,1).
\]
Hence coordinate changes in the observable and action lifts act on \(M(\tau)\) by invertible congruence. The coordinate-free content of the construction is therefore the congruence class generated by these allowed gauge transformations, rather than a particular matrix representative.
\end{proof}

\subsection{Proof of Theorem~\ref{thm:traj-complete}}

\begin{proof}
Throughout the proof, the Frobenius inner product on matrices is
\[
\langle A,B\rangle
:=
\operatorname{tr}(A^\top B).
\]
For vectors \(u,v\in\mathbb R^m\) and a matrix \(A\in\mathbb R^{m\times m}\), we shall repeatedly use the identity
\[
\langle A,uv^\top\rangle
=
\operatorname{tr}(A^\top uv^\top)
=
\operatorname{tr}(v^\top A^\top u)
=
v^\top A^\top u.
\]
In particular, if \(A=A^\top\) and \(u=v=\psi_t\), then
\[
\langle A,\psi_t\psi_t^\top\rangle
=
\psi_t^\top A\psi_t.
\]

Let \(F\in\mathfrak F_{\mathrm{aq}}\). By Definition~\ref{def:qagc}, there exist coefficients
\[
a\in\mathbb R,\qquad b\in\mathbb R^{m-1},\qquad C\in\mathbb S^{m-1}
\]
such that, for every finite segment \(\tau\),
\[
F(\tau)
=
\sum_{t=0}^{T-1}
\left(
a+b^\top\bar\psi_t+\bar\psi_t^\top C\bar\psi_t
\right),
\qquad
\psi_t=
\begin{bmatrix}
\bar\psi_t\\
1
\end{bmatrix}.
\]
Define the symmetric block matrix
\[
A
:=
\begin{bmatrix}
C & \frac12 b\\
\frac12 b^\top & a
\end{bmatrix}
\in\mathbb S^m.
\]
For each time index \(t\), direct block multiplication gives
\[
A\psi_t
=
\begin{bmatrix}
C & \frac12 b\\
\frac12 b^\top & a
\end{bmatrix}
\begin{bmatrix}
\bar\psi_t\\
1
\end{bmatrix}
=
\begin{bmatrix}
C\bar\psi_t+\frac12 b\\
\frac12 b^\top\bar\psi_t+a
\end{bmatrix}.
\]
Therefore
\[
\psi_t^\top A\psi_t
=
\begin{bmatrix}
\bar\psi_t^\top & 1
\end{bmatrix}
\begin{bmatrix}
C\bar\psi_t+\frac12 b\\
\frac12 b^\top\bar\psi_t+a
\end{bmatrix}
=
\bar\psi_t^\top C\bar\psi_t
+
\frac12\bar\psi_t^\top b
+
\frac12 b^\top\bar\psi_t
+
a.
\]
Since \(\bar\psi_t^\top b=b^\top\bar\psi_t\), the two linear half-terms combine:
\[
\psi_t^\top A\psi_t
=
\bar\psi_t^\top C\bar\psi_t
+
b^\top\bar\psi_t
+
a.
\]
Using the identity \(\langle A,\psi_t\psi_t^\top\rangle=\psi_t^\top A\psi_t\), we obtain
\[
\langle A,\psi_t\psi_t^\top\rangle
=
a+b^\top\bar\psi_t+\bar\psi_t^\top C\bar\psi_t.
\]
Summing this equality over \(t\) yields
\[
\begin{aligned}
\langle A,M(\tau)\rangle
&=
\left\langle A,\sum_{t=0}^{T-1}\psi_t\psi_t^\top\right\rangle\\
&=
\sum_{t=0}^{T-1}\langle A,\psi_t\psi_t^\top\rangle\\
&=
\sum_{t=0}^{T-1}
\left(
a+b^\top\bar\psi_t+\bar\psi_t^\top C\bar\psi_t
\right)\\
&=
F(\tau).
\end{aligned}
\]
Thus every functional in \(\mathfrak F_{\mathrm{aq}}\) is represented by a linear functional of \(M(\tau)\).

Conversely, let \(A\in\mathbb S^m\) be arbitrary. Partition \(A\) according to the decomposition \(\mathbb R^m=\mathbb R^{m-1}\oplus\mathbb R\):
\[
A=
\begin{bmatrix}
C & u\\
u^\top & \alpha
\end{bmatrix},
\qquad
C\in\mathbb R^{(m-1)\times(m-1)},\quad
u\in\mathbb R^{m-1},\quad
\alpha\in\mathbb R.
\]
Since \(A\) is symmetric, its upper-left block satisfies \(C=C^\top\), so \(C\in\mathbb S^{m-1}\). For each \(t\),
\[
A\psi_t
=
\begin{bmatrix}
C & u\\
u^\top & \alpha
\end{bmatrix}
\begin{bmatrix}
\bar\psi_t\\
1
\end{bmatrix}
=
\begin{bmatrix}
C\bar\psi_t+u\\
u^\top\bar\psi_t+\alpha
\end{bmatrix},
\]
and hence
\[
\psi_t^\top A\psi_t
=
\begin{bmatrix}
\bar\psi_t^\top & 1
\end{bmatrix}
\begin{bmatrix}
C\bar\psi_t+u\\
u^\top\bar\psi_t+\alpha
\end{bmatrix}
=
\bar\psi_t^\top C\bar\psi_t
+
\bar\psi_t^\top u
+
u^\top\bar\psi_t
+
\alpha.
\]
Because \(\bar\psi_t^\top u=u^\top\bar\psi_t\), this becomes
\[
\psi_t^\top A\psi_t
=
\bar\psi_t^\top C\bar\psi_t
+
2u^\top\bar\psi_t
+
\alpha.
\]
Therefore
\[
\begin{aligned}
\langle A,M(\tau)\rangle
&=
\sum_{t=0}^{T-1}\langle A,\psi_t\psi_t^\top\rangle\\
&=
\sum_{t=0}^{T-1}\psi_t^\top A\psi_t\\
&=
\sum_{t=0}^{T-1}
\left(
\alpha+2u^\top\bar\psi_t+\bar\psi_t^\top C\bar\psi_t
\right).
\end{aligned}
\]
This is a member of \(\mathfrak F_{\mathrm{aq}}\), with coefficients
\[
a=\alpha,\qquad b=2u,\qquad C=C.
\]
Thus every symmetric matrix \(A\) defines a functional in the low-order additive class through the formula \(F_A(\tau)=\langle A,M(\tau)\rangle\).

It remains to prove the equivalence statement. Suppose first that \(M(\tau)=M(\tau')\). Let \(F\in\mathfrak F_{\mathrm{aq}}\) be arbitrary. By the first part of the proof, there exists \(A\in\mathbb S^m\) such that
\[
F(\eta)=\langle A,M(\eta)\rangle
\]
for every finite segment \(\eta\). Applying this to \(\tau\) and \(\tau'\), and using \(M(\tau)=M(\tau')\), gives
\[
F(\tau)
=
\langle A,M(\tau)\rangle
=
\langle A,M(\tau')\rangle
=
F(\tau').
\]
Hence equality of matrix proxies implies equality of all functionals in \(\mathfrak F_{\mathrm{aq}}\).

Conversely, suppose
\[
F(\tau)=F(\tau')
\qquad
\text{for every }F\in\mathfrak F_{\mathrm{aq}}.
\]
The converse representation already proved implies that, for every \(A\in\mathbb S^m\), the functional
\[
F_A(\eta):=\langle A,M(\eta)\rangle
\]
belongs to \(\mathfrak F_{\mathrm{aq}}\). Applying the assumed equality to \(F_A\) gives
\[
\langle A,M(\tau)\rangle
=
\langle A,M(\tau')\rangle
\qquad
\text{for every }A\in\mathbb S^m.
\]
Equivalently,
\[
\langle A,M(\tau)-M(\tau')\rangle=0
\qquad
\text{for every }A\in\mathbb S^m.
\]
By Theorem~\ref{thm:traj-well-defined}, both \(M(\tau)\) and \(M(\tau')\) are symmetric, so the difference
\[
\Delta:=M(\tau)-M(\tau')
\]
also belongs to \(\mathbb S^m\). Choosing \(A=\Delta\) in the preceding identity gives
\[
0
=
\langle \Delta,\Delta\rangle
=
\operatorname{tr}(\Delta^\top\Delta)
=
\|\Delta\|_F^2.
\]
A Frobenius norm is zero if and only if the matrix is zero. Hence \(\Delta=0\), and therefore
\[
M(\tau)=M(\tau').
\]
This proves both directions of the equivalence and shows that \(M(\tau)\) captures exactly the additive affine-quadratic information induced by the fixed lift.
\end{proof}

\subsection{Proof of Theorem~\ref{thm:traj-composition}}

\begin{proof}
Let
\[
\tau^{(1)}
=
((o_t^{(1)},a_t^{(1)},r_t^{(1)},o_{t+1}^{(1)}))_{t=0}^{T_1-1}
\]
and
\[
\tau^{(2)}
=
((o_s^{(2)},a_s^{(2)},r_s^{(2)},o_{s+1}^{(2)}))_{s=0}^{T_2-1}
\]
be finite trajectory segments. Suppose
\[
\tau=\tau^{(1)}\star\tau^{(2)}
\]
is a valid concatenation. Validity means that the terminal observation of the first segment matches the initial observation of the second segment, so the raw transitions can be placed one after another without inserting any additional transition at the boundary. The concatenated segment has length \(T_1+T_2\).

Let \(\psi_t^{(1)}\) be the lifted vector associated with the \(t\)-th transition of \(\tau^{(1)}\), and let \(\psi_s^{(2)}\) be the lifted vector associated with the \(s\)-th transition of \(\tau^{(2)}\). Since the fixed lift \((\phi,\xi)\) is applied transition by transition, the lifted vectors of the concatenated trajectory are precisely
\[
\psi_t
=
\psi_t^{(1)}
\qquad
\text{for }0\le t\le T_1-1,
\]
and
\[
\psi_{T_1+s}
=
\psi_s^{(2)}
\qquad
\text{for }0\le s\le T_2-1.
\]
No extra outer product appears at the joining point, because the descriptor sums over actual transitions and the concatenation operation does not create a new transition between the two pieces.

By Definition~\ref{def:traj-matrix-final},
\[
M(\tau)
=
\sum_{t=0}^{T_1+T_2-1}\psi_t\psi_t^\top.
\]
Splitting the finite sum at \(T_1\) gives
\[
M(\tau)
=
\sum_{t=0}^{T_1-1}\psi_t\psi_t^\top
+
\sum_{t=T_1}^{T_1+T_2-1}\psi_t\psi_t^\top.
\]
Using the identities for the lifted vectors on the first and second parts, the first sum becomes
\[
\sum_{t=0}^{T_1-1}\psi_t\psi_t^\top
=
\sum_{t=0}^{T_1-1}\psi_t^{(1)}(\psi_t^{(1)})^\top
=
M(\tau^{(1)}),
\]
while the second sum, after the change of index \(t=T_1+s\), becomes
\[
\sum_{t=T_1}^{T_1+T_2-1}\psi_t\psi_t^\top
=
\sum_{s=0}^{T_2-1}\psi_{T_1+s}\psi_{T_1+s}^\top
=
\sum_{s=0}^{T_2-1}\psi_s^{(2)}(\psi_s^{(2)})^\top
=
M(\tau^{(2)}).
\]
Combining these two identities yields
\[
M(\tau)
=
M(\tau^{(1)})+M(\tau^{(2)}).
\]
Thus valid concatenation of trajectory segments is represented exactly as addition in the common matrix space.
\end{proof}

\subsection{Proof of Theorem~\ref{thm:traj-gluing}}

\begin{proof}
Let \(\tau=((o_t,a_t,r_t,o_{t+1}))_{t=0}^{T-1}\) be fixed, and let \(\psi_t\) denote the lifted vector associated with its \(t\)-th transition. For an interval \(I\subseteq\{0,\ldots,T-1\}\), the restricted trajectory \(\tau|_I\) consists exactly of those transitions whose indices lie in \(I\). Hence its matrix descriptor is
\[
M(\tau|_I)
=
\sum_{t\in I}\psi_t\psi_t^\top.
\]
This identity is simply the definition of \(M\) applied to the restricted segment; the same ambient lift is used, so no change of dimension or coordinates occurs when passing from \(\tau\) to \(\tau|_I\).

Let \(\{I_k\}_{k=1}^K\) be an interval cover of \(\{0,\ldots,T-1\}\). For each nonempty \(S\subseteq\{1,\ldots,K\}\), define
\[
I_S=\bigcap_{k\in S}I_k.
\]
The theorem restricts attention to the index family
\[
\mathcal N
=
\{S\subseteq\{1,\ldots,K\}:S\neq\varnothing,\ I_S\neq\varnothing\}.
\]
For \(S\in\mathcal N\), the descriptor of the overlap restriction is
\[
M(\tau_S)
=
M(\tau|_{I_S})
=
\sum_{t\in I_S}\psi_t\psi_t^\top.
\]
Therefore the inclusion--exclusion expression equals
\[
\sum_{S\in\mathcal N}(-1)^{|S|+1}M(\tau_S)
=
\sum_{S\in\mathcal N}(-1)^{|S|+1}
\sum_{t\in I_S}\psi_t\psi_t^\top.
\]
All sums are finite, so the order of summation may be interchanged. This gives
\[
\sum_{S\in\mathcal N}(-1)^{|S|+1}M(\tau_S)
=
\sum_{t=0}^{T-1}
\left(
\sum_{\substack{S\in\mathcal N\\ t\in I_S}}
(-1)^{|S|+1}
\right)
\psi_t\psi_t^\top.
\]
It remains to compute the scalar coefficient multiplying each fixed transition contribution \(\psi_t\psi_t^\top\).

Fix \(t\in\{0,\ldots,T-1\}\). Since \(\{I_k\}_{k=1}^K\) covers \(\{0,\ldots,T-1\}\), the set
\[
K(t)
:=
\{k\in\{1,\ldots,K\}:t\in I_k\}
\]
is nonempty. For any nonempty \(S\subseteq\{1,\ldots,K\}\), we have
\[
t\in I_S
\quad\Longleftrightarrow\quad
t\in\bigcap_{k\in S}I_k
\quad\Longleftrightarrow\quad
t\in I_k\text{ for every }k\in S
\quad\Longleftrightarrow\quad
S\subseteq K(t).
\]
Moreover, whenever \(S\subseteq K(t)\) and \(S\neq\varnothing\), the intersection \(I_S\) is nonempty because it contains \(t\). Hence such an \(S\) automatically belongs to \(\mathcal N\). Thus
\[
\{S\in\mathcal N:t\in I_S\}
=
\{S:S\neq\varnothing,\ S\subseteq K(t)\}.
\]
The coefficient of \(\psi_t\psi_t^\top\) is therefore
\[
\sum_{\substack{S\in\mathcal N\\ t\in I_S}}
(-1)^{|S|+1}
=
\sum_{\varnothing\neq S\subseteq K(t)}
(-1)^{|S|+1}.
\]
Because \(K(t)\) is nonempty, the binomial identity gives
\[
\sum_{S\subseteq K(t)}(-1)^{|S|}
=
(1-1)^{|K(t)|}
=
0.
\]
Separating the empty subset term from this sum yields
\[
1+\sum_{\varnothing\neq S\subseteq K(t)}(-1)^{|S|}
=
0.
\]
Multiplying by \(-1\), we obtain
\[
\sum_{\varnothing\neq S\subseteq K(t)}(-1)^{|S|+1}
=
1.
\]
Thus every transition contribution \(\psi_t\psi_t^\top\) appears with coefficient exactly one in the inclusion--exclusion reconstruction. Substituting this coefficient back into the reordered sum gives
\[
\sum_{S\in\mathcal N}(-1)^{|S|+1}M(\tau_S)
=
\sum_{t=0}^{T-1}\psi_t\psi_t^\top
=
M(\tau).
\]
This proves exact local-to-global recovery for an actual global trajectory.

For the more general representation-level statement, suppose an arbitrary family
\[
\{M_S\}_{S\in\mathcal N}\subseteq\mathbb S^m
\]
is given. Since \(\mathbb S^m\) is a real vector space, every finite linear combination of elements of \(\mathbb S^m\) is again an element of \(\mathbb S^m\). Therefore
\[
\widehat M
:=
\sum_{S\in\mathcal N}(-1)^{|S|+1}M_S
\]
is a well-defined element of \(\mathbb S^m\). Its uniqueness follows from the fact that finite sums in a vector space are unique: once the family \(\{M_S\}_{S\in\mathcal N}\) and the coefficients \((-1)^{|S|+1}\) are fixed, there is exactly one matrix equal to that linear combination. This establishes the candidate glued matrix at the representation level. The argument does not imply that \(\widehat M\) is positive semidefinite, that its constant channel is a valid horizon, or that it is reachable by an environment trajectory; those are separate realizability properties handled by the obstruction theorem.
\end{proof}

\subsection{Proof of Theorem~\ref{thm:traj-minimal}}

\begin{proof}
We use the factorization meaning of admissible sufficiency in Definition~\ref{def:admissible-descriptor}. Thus, if \(D\) is admissible, then for every \(F\in\mathfrak F_{\mathrm{aq}}\) there exists a Borel measurable map
\[
h_F:\mathbb R^p\to\mathbb R
\]
such that
\[
F(\tau)=h_F(D(\tau))
\qquad
\text{for every finite trajectory segment }\tau.
\]
This measurable factorization assumption is the point that allows one to construct the measurable post-processing map \(\rho\).

Let
\[
s:=\dim(\mathbb S^m)=\frac{m(m+1)}{2}.
\]
Equip \(\mathbb S^m\) with the Frobenius inner product
\[
\langle A,B\rangle=\operatorname{tr}(A^\top B).
\]
Since \(\mathbb S^m\) is a finite-dimensional real Hilbert space under this inner product, choose an orthonormal basis
\[
B_1,\ldots,B_s
\]
of \(\mathbb S^m\). Thus every \(N\in\mathbb S^m\) has the unique expansion
\[
N=\sum_{j=1}^s \langle B_j,N\rangle B_j,
\]
and the coordinates \(\langle B_j,N\rangle\) determine \(N\) uniquely.

For each \(j\in\{1,\ldots,s\}\), define a trajectory functional
\[
F_j(\tau):=\langle B_j,M(\tau)\rangle.
\]
By Theorem~\ref{thm:traj-complete}, every functional of the form \(\tau\mapsto \langle A,M(\tau)\rangle\), with \(A\in\mathbb S^m\), belongs to \(\mathfrak F_{\mathrm{aq}}\). Taking \(A=B_j\), we obtain
\[
F_j\in\mathfrak F_{\mathrm{aq}}
\qquad
\text{for every }j.
\]
Since \(D\) is admissible and therefore sufficient for all functionals in \(\mathfrak F_{\mathrm{aq}}\), there exists a Borel measurable map
\[
h_j:\mathbb R^p\to\mathbb R
\]
such that
\[
F_j(\tau)=h_j(D(\tau))
\qquad
\text{for every }\tau.
\]
Equivalently,
\[
\langle B_j,M(\tau)\rangle
=
h_j(D(\tau))
\qquad
\text{for every }\tau
\]
and for every \(j\in\{1,\ldots,s\}\).

Define
\[
\rho:\mathbb R^p\to\mathbb S^m
\]
by
\[
\rho(y)
:=
\sum_{j=1}^s h_j(y)B_j.
\]
This map is Borel measurable. Indeed, identify \(\mathbb S^m\) with \(\mathbb R^s\) through its coordinates in the basis \(B_1,\ldots,B_s\). Under this identification, \(\rho\) has coordinate functions
\[
y\mapsto h_1(y),\ldots,y\mapsto h_s(y),
\]
all of which are Borel measurable. Since a map into a finite-dimensional Euclidean space is Borel measurable if and only if its coordinate functions are Borel measurable, \(\rho\) is Borel measurable.

We now show that \(\rho\) recovers \(M\) from \(D\). Fix any trajectory segment \(\tau\). For each \(i\in\{1,\ldots,s\}\),
\[
\begin{aligned}
\langle B_i,\rho(D(\tau))\rangle
&=
\left\langle
B_i,\sum_{j=1}^s h_j(D(\tau))B_j
\right\rangle\\
&=
\sum_{j=1}^s h_j(D(\tau))\langle B_i,B_j\rangle.
\end{aligned}
\]
Because \(B_1,\ldots,B_s\) is orthonormal,
\[
\langle B_i,B_j\rangle=\delta_{ij},
\]
and hence
\[
\langle B_i,\rho(D(\tau))\rangle
=
h_i(D(\tau)).
\]
By construction of \(h_i\),
\[
h_i(D(\tau))
=
F_i(\tau)
=
\langle B_i,M(\tau)\rangle.
\]
Therefore
\[
\langle B_i,\rho(D(\tau))\rangle
=
\langle B_i,M(\tau)\rangle
\qquad
\text{for every }i.
\]
The two symmetric matrices \(\rho(D(\tau))\) and \(M(\tau)\) have identical coordinates in the orthonormal basis \(B_1,\ldots,B_s\). Hence they are equal:
\[
\rho(D(\tau))=M(\tau).
\]
Since \(\tau\) was arbitrary, this proves
\[
M(\tau)=\rho(D(\tau))
\qquad
\text{for every finite trajectory segment }\tau.
\]

It remains to justify the minimal-sufficiency interpretation. Strictly speaking, the admissible descriptors in Definition~\ref{def:admissible-descriptor} take values in Euclidean spaces, whereas \(M\) takes values in \(\mathbb S^m\). This is only a notational difference, because the vector space \(\mathbb S^m\) is canonically identifiable with \(\mathbb R^s\), for example by the half-vectorization map \(\operatorname{vech}\). Under this identification, \(M\) is a fixed-dimensional descriptor.

The descriptor \(M\) is additive under valid concatenation by Theorem~\ref{thm:traj-composition}. It is sufficient for \(\mathfrak F_{\mathrm{aq}}\) by Theorem~\ref{thm:traj-complete}: if \(F\in\mathfrak F_{\mathrm{aq}}\), then there exists \(A\in\mathbb S^m\) such that
\[
F(\tau)=\langle A,M(\tau)\rangle.
\]
The map
\[
N\mapsto \langle A,N\rangle
\]
is linear, hence continuous, hence Borel measurable. Therefore \(F\) factors measurably through \(M\).

Finally, the factorization proved above says that for every admissible descriptor \(D\), the matrix proxy \(M\) is obtained by measurable post-processing of \(D\). In the usual sufficiency preorder, a descriptor \(D_1\) is no finer than a descriptor \(D_2\) when \(D_1\) is a measurable function of \(D_2\). The identity
\[
M=\rho\circ D
\]
therefore says that every admissible descriptor \(D\) is at least as fine as \(M\). Since \(M\) itself is admissible, \(M\) is the coarsest admissible additive descriptor preserving all functionals in \(\mathfrak F_{\mathrm{aq}}\), up to measurable post-processing.
\end{proof}

\subsection{Proof of Theorem~\ref{thm:traj-obstruction}}

\begin{proof}
Let \(e_{\mathrm{c}}\in\mathbb R^m\) denote the standard basis vector corresponding to the constant channel of the lift \(\psi_t\), namely the last coordinate in
\[
\psi_t=
\begin{bmatrix}
c_t\\
\delta_t\\
\xi(a_t)\otimes c_t\\
\xi(a_t)\otimes\delta_t\\
r_t\\
1
\end{bmatrix}.
\]
For any trajectory segment \(\tau\) of horizon \(T\), the constant-channel entry of \(M(\tau)\) is
\[
e_{\mathrm{c}}^\top M(\tau)e_{\mathrm{c}}
=
e_{\mathrm{c}}^\top
\left(\sum_{t=0}^{T-1}\psi_t\psi_t^\top\right)
e_{\mathrm{c}}
=
\sum_{t=0}^{T-1}
(e_{\mathrm{c}}^\top\psi_t)^2.
\]
Since \(e_{\mathrm{c}}^\top\psi_t=1\) for every \(t\), this becomes
\[
e_{\mathrm{c}}^\top M(\tau)e_{\mathrm{c}}
=
\sum_{t=0}^{T-1}1
=
T.
\]
Thus the constant-channel block records the horizon for every realizable matrix. The theorem is applied to a candidate \(\widehat M\) whose constant-channel block encodes the same horizon \(T\). If the encoded value is not a nonnegative integer, then it cannot equal the horizon of any finite trajectory segment, and non-realizability is immediate. Hence the only nontrivial case is the stated case in which the encoded value is a valid horizon \(T\in\mathbb N_0\).

Recall the reachable set
\[
\mathcal R_e^{(T)}
=
\{M(\tau):\tau\text{ is realizable in }e\text{ with horizon }T\}.
\]
The obstruction score is
\[
\operatorname{Obs}_e^{(T)}(\widehat M)
=
\inf_{M\in\mathcal R_e^{(T)}}\|\widehat M-M\|_F,
\]
with the standard convention that the infimum over an empty set is \(+\infty\). If \(\mathcal R_e^{(T)}=\varnothing\), then no trajectory of horizon \(T\) is realizable in environment \(e\). In that case \(\widehat M\) cannot correspond to any realizable trajectory at horizon \(T\), so the conclusion holds.

Assume now that \(\mathcal R_e^{(T)}\neq\varnothing\). Suppose, for contradiction, that \(\widehat M\) corresponds to a realizable trajectory in environment \(e\) at horizon \(T\). Then there exists a trajectory segment \(\tau\), realizable in \(e\) with horizon \(T\), such that
\[
\widehat M=M(\tau).
\]
By the definition of \(\mathcal R_e^{(T)}\), this implies
\[
M(\tau)\in\mathcal R_e^{(T)}.
\]
The distance from \(\widehat M\) to the reachable set is bounded above by its distance to any particular reachable element. Therefore
\[
\operatorname{Obs}_e^{(T)}(\widehat M)
=
\inf_{M\in\mathcal R_e^{(T)}}\|\widehat M-M\|_F
\le
\|\widehat M-M(\tau)\|_F.
\]
Since \(\widehat M=M(\tau)\), the right-hand side is
\[
\|\widehat M-M(\tau)\|_F
=
\|0\|_F
=
0.
\]
The Frobenius norm is nonnegative, so the infimum of such nonnegative distances is also nonnegative:
\[
\operatorname{Obs}_e^{(T)}(\widehat M)\ge 0.
\]
Combining the two inequalities gives
\[
\operatorname{Obs}_e^{(T)}(\widehat M)=0.
\]
This contradicts the assumption
\[
\operatorname{Obs}_e^{(T)}(\widehat M)>0.
\]
Therefore no realizable trajectory in environment \(e\) at horizon \(T\) can have matrix proxy equal to \(\widehat M\). Positive obstruction is consequently a certificate of inconsistency with the environment. The argument proves only this one-way implication; it does not assert that zero obstruction implies exact realizability unless additional closedness or attainment assumptions are imposed on \(\mathcal R_e^{(T)}\).
\end{proof}

\subsection{Proof of Proposition~\ref{prop:matrix-bellman-reduction}}

\begin{proof}
By definition of \(\mathcal T_M\),
\[
(\mathcal T_M \bar Q^\star)(o,M,a)
=
\mathbb E\!\left[
r+\gamma\max_{a'\in\mathcal A}
\bar Q^\star\!\bigl(o',M+\psi(o,a,r,o')\psi(o,a,r,o')^\top,a'\bigr)
\mid o,M,a
\right].
\]
Assumption~\ref{ass:matrix-bellman-compatibility} states exactly that the right-hand side equals \(\bar Q^\star(o,M,a)\) whenever the lifted transition law is well defined. Hence
\[
\bar Q^\star(o,M,a)=(\mathcal T_M\bar Q^\star)(o,M,a)
\]
pointwise, which is the fixed-point identity \(\bar Q^\star=\mathcal T_M\bar Q^\star\).

Next, Assumption~\ref{ass:matrix-conditioned-value} states that for every environment \(e\) in the aligned family and every rollout prefix \(\tau_{0:t-1}\),
\[
Q_e^\star(o_t,\tau_{0:t-1},a)
=
\bar Q^\star(o_t,M(\tau_{0:t-1}),a).
\]
Thus the optimal action-value of the original history-dependent problem is obtained by evaluating the lifted Bellman fixed point at the matrix state generated by the prefix. Therefore the Bellman recursion closes exactly on the lifted state \((o,M)\).
\end{proof}

\subsection{Proof of Theorem~\ref{thm:matrix-bellman-linearization}}

\begin{proof}
Fix \(o\in\mathcal O\), \(a\in\mathcal A\), and a reference matrix \(M\) in the relevant matrix domain. Assumptions~\ref{ass:matrix-conditioned-value} and~\ref{ass:matrix-bellman-compatibility} identify \(\bar Q^\star(o,M,a)\) as the optimal Bellman value associated with the lifted state \((o,M)\). The analytic expansion itself uses Assumption~\ref{ass:matrix-value-smoothness}.

By Assumption~\ref{ass:matrix-value-smoothness}, there exist an open set
\[
\mathcal U_{o,a}\subseteq\mathbb S^m
\]
containing \(M\) and a function
\[
\widehat Q^\star_{o,a}:\mathcal U_{o,a}\to\mathbb R
\]
such that
\[
\widehat Q^\star_{o,a}(N)=\bar Q^\star(o,N,a)
\]
for every relevant \(N\in\mathcal U_{o,a}\). The map \(\widehat Q^\star_{o,a}\) is Fr\'echet differentiable at \(M\). Therefore there exists a bounded linear functional
\[
D\widehat Q^\star_{o,a}(M):\mathbb S^m\to\mathbb R
\]
such that
\[
\widehat Q^\star_{o,a}(M+K)
=
\widehat Q^\star_{o,a}(M)+D\widehat Q^\star_{o,a}(M)[K]+r(K)
\]
for all sufficiently small \(K\in\mathbb S^m\), where the remainder \(r(K)\) satisfies
\[
\frac{|r(K)|}{\|K\|_F}\to 0
\qquad
\text{as }\|K\|_F\to0.
\]
Equivalently,
\[
r(K)=o(\|K\|_F)
\qquad
\text{as }\|K\|_F\to0.
\]

The vector space \(\mathbb S^m\), equipped with the Frobenius inner product, is a finite-dimensional Hilbert space. Hence every bounded linear functional on \(\mathbb S^m\) has a unique Riesz representative in \(\mathbb S^m\). Therefore there exists a unique matrix
\[
W_{o,M,a}\in\mathbb S^m
\]
such that
\[
D\widehat Q^\star_{o,a}(M)[K]
=
\langle W_{o,M,a},K\rangle
\qquad
\text{for every }K\in\mathbb S^m.
\]
Substituting this representation into the Fr\'echet expansion gives
\[
\widehat Q^\star_{o,a}(M+K)
=
\widehat Q^\star_{o,a}(M)+\langle W_{o,M,a},K\rangle+r(K).
\]

Now take
\[
K=H_0+H_1.
\]
Whenever \(\|H_0\|_F+\|H_1\|_F\) is sufficiently small and \(M+H_0+H_1\) lies in the relevant domain, we have \(M+H_0+H_1\in\mathcal U_{o,a}\) and
\[
\bar Q^\star(o,M+H_0+H_1,a)
=
\widehat Q^\star_{o,a}(M+H_0+H_1).
\]
Thus
\[
\bar Q^\star(o,M+H_0+H_1,a)
=
\widehat Q^\star_{o,a}(M)+\langle W_{o,M,a},H_0+H_1\rangle+r(H_0+H_1).
\]
Since \(\widehat Q^\star_{o,a}(M)=\bar Q^\star(o,M,a)\), this becomes
\[
\bar Q^\star(o,M+H_0+H_1,a)
=
\bar Q^\star(o,M,a)
+
\langle W_{o,M,a},H_0+H_1\rangle
+
r(H_0+H_1).
\]
The remainder satisfies
\[
r(H_0+H_1)
=
o(\|H_0+H_1\|_F)
\qquad
\text{as }\|H_0+H_1\|_F\to0.
\]
Therefore
\[
\bar Q^\star(o,M+H_0+H_1,a)
=
\bar Q^\star(o,M,a)
+
\langle W_{o,M,a},H_0+H_1\rangle
+
o(\|H_0+H_1\|_F).
\]
This proves the first claimed expansion.

We now prove the approximate additivity of value increments. Define, for sufficiently small \(H\),
\[
\Delta(H)
:=
\bar Q^\star(o,M+H,a)-\bar Q^\star(o,M,a).
\]
Using the expansion above with \(K=H\), we have
\[
\Delta(H)
=
\langle W_{o,M,a},H\rangle+r(H),
\qquad
r(H)=o(\|H\|_F).
\]
Hence
\[
\Delta(H_0+H_1)
=
\langle W_{o,M,a},H_0+H_1\rangle+r(H_0+H_1),
\]
while
\[
\Delta(H_0)
=
\langle W_{o,M,a},H_0\rangle+r(H_0)
\]
and
\[
\Delta(H_1)
=
\langle W_{o,M,a},H_1\rangle+r(H_1).
\]
Because the Frobenius inner product is linear in its second argument,
\[
\langle W_{o,M,a},H_0+H_1\rangle
=
\langle W_{o,M,a},H_0\rangle
+
\langle W_{o,M,a},H_1\rangle.
\]
Subtracting the expressions gives
\[
\Delta(H_0+H_1)-\Delta(H_0)-\Delta(H_1)
=
r(H_0+H_1)-r(H_0)-r(H_1).
\]

Let
\[
s(H_0,H_1):=\|H_0\|_F+\|H_1\|_F.
\]
We claim that
\[
r(H_0+H_1)-r(H_0)-r(H_1)
=
o(s(H_0,H_1))
\qquad
\text{as }s(H_0,H_1)\to0.
\]
Let \(\varepsilon>0\). Since \(r(H)=o(\|H\|_F)\), there exists \(\delta>0\) such that
\[
\|H\|_F<\delta
\quad\Longrightarrow\quad
|r(H)|\le \varepsilon\|H\|_F.
\]
If
\[
s(H_0,H_1)=\|H_0\|_F+\|H_1\|_F<\delta,
\]
then
\[
\|H_0\|_F<\delta,\qquad
\|H_1\|_F<\delta,
\]
and, by the triangle inequality,
\[
\|H_0+H_1\|_F
\le
\|H_0\|_F+\|H_1\|_F
=
s(H_0,H_1)
<
\delta.
\]
Therefore
\[
|r(H_0)|\le\varepsilon\|H_0\|_F,\qquad
|r(H_1)|\le\varepsilon\|H_1\|_F,
\]
and
\[
|r(H_0+H_1)|
\le
\varepsilon\|H_0+H_1\|_F
\le
\varepsilon(\|H_0\|_F+\|H_1\|_F).
\]
Hence
\[
\begin{aligned}
&|r(H_0+H_1)-r(H_0)-r(H_1)|\\
&\qquad\le
|r(H_0+H_1)|+|r(H_0)|+|r(H_1)|\\
&\qquad\le
2\varepsilon(\|H_0\|_F+\|H_1\|_F).
\end{aligned}
\]
Since \(\varepsilon>0\) was arbitrary,
\[
r(H_0+H_1)-r(H_0)-r(H_1)
=
o(\|H_0\|_F+\|H_1\|_F).
\]

Returning to \(\Delta\), we obtain
\[
\Delta(H_0+H_1)
=
\Delta(H_0)+\Delta(H_1)
+
o(\|H_0\|_F+\|H_1\|_F).
\]
Expanding the definition of \(\Delta\) gives
\[
\begin{aligned}
&\bar Q^\star(o,M+H_0+H_1,a)-\bar Q^\star(o,M,a)\\
&\qquad =
\bigl[\bar Q^\star(o,M+H_0,a)-\bar Q^\star(o,M,a)\bigr]
+
\bigl[\bar Q^\star(o,M+H_1,a)-\bar Q^\star(o,M,a)\bigr]\\
&\qquad\quad
+
o(\|H_0\|_F+\|H_1\|_F).
\end{aligned}
\]
This is the claimed first-order additivity of Bellman-value increments in matrix coordinates. The centered approximation in the theorem is the same identity specialized to the base point \(M=0\), with the common base value subtracted once.
\end{proof}

\subsection{Proof of Corollary~\ref{cor:centered-matrix-value}}

\begin{proof}
Fix \(o\in\mathcal O\) and \(a\in\mathcal A\). Assume that \(0\), \(H_0\), \(H_1\), and \(H_0+H_1\) lie in the relevant matrix domain, and that the increments are sufficiently small for the local expansion in Theorem~\ref{thm:matrix-bellman-linearization} to apply at the reference matrix \(M=0\). The centered value is defined by
\[
\widetilde Q^\star(o,M,a)
=
\bar Q^\star(o,M,a)-\bar Q^\star(o,0,a).
\]
Applying Theorem~\ref{thm:matrix-bellman-linearization} with reference point \(M=0\) gives
\[
\begin{aligned}
&\bar Q^\star(o,H_0+H_1,a)-\bar Q^\star(o,0,a)\\
&\qquad =
\bigl[\bar Q^\star(o,H_0,a)-\bar Q^\star(o,0,a)\bigr]
+
\bigl[\bar Q^\star(o,H_1,a)-\bar Q^\star(o,0,a)\bigr]\\
&\qquad\quad
+
o(\|H_0\|_F+\|H_1\|_F).
\end{aligned}
\]
The left-hand side is exactly
\[
\widetilde Q^\star(o,H_0+H_1,a).
\]
The first bracket on the right-hand side is
\[
\bar Q^\star(o,H_0,a)-\bar Q^\star(o,0,a)
=
\widetilde Q^\star(o,H_0,a),
\]
and the second bracket is
\[
\bar Q^\star(o,H_1,a)-\bar Q^\star(o,0,a)
=
\widetilde Q^\star(o,H_1,a).
\]
Substituting these three identities into the theorem's expansion yields
\[
\widetilde Q^\star(o,H_0+H_1,a)
=
\widetilde Q^\star(o,H_0,a)
+
\widetilde Q^\star(o,H_1,a)
+
o(\|H_0\|_F+\|H_1\|_F).
\]
This proves the centered compositional value law.
\end{proof}

\subsection{Additional Matrix-Value Consequences}
\label{app:value-consequences}

\begin{proposition}[Local value stability and value-null directions]
\label{prop:matrix-value-stability}
Fix \((o,a)\) and a reference matrix \(M\). Under
Assumption~\ref{ass:matrix-value-smoothness}, there exist a Frobenius
neighborhood \(\mathcal U\) of \(M\) in the relevant matrix domain and a constant
\(L_{o,M,a}<\infty\) such that
$
|\bar Q^\star(o,N_1,a)-\bar Q^\star(o,N_0,a)|
\le
L_{o,M,a}\|N_1-N_0\|_F
\qquad
\text{for all }N_0,N_1\in\mathcal U.
$
Moreover, if \(H\in\mathbb S^m\) is such that \(M+H\in\mathcal U\), then
$
\bar Q^\star(o,M+H,a)-\bar Q^\star(o,M,a)
=
\langle W_{o,M,a},H\rangle
+
o(\|H\|_F).
$
In particular, any perturbation direction \(H\) satisfying
$
\langle W_{o,M,a},H\rangle=0
$
is invisible to first order:
$
\bar Q^\star(o,M+H,a)-\bar Q^\star(o,M,a)
=
o(\|H\|_F).
$
\end{proposition}
Applied to a grounded transition structure \(G\) with descriptor \(M(G)\),
Proposition~\ref{prop:matrix-value-stability} says that exact
matrix-preserving transformations leave value unchanged exactly, while
families of deformations with small induced matrix perturbation that is
Frobenius-orthogonal to the local gradient are value-invisible to first
order. Equivalently, the matrix directions aligned with \(W_{o,M,a}\) are
the locally value-relevant structural features, while Frobenius-orthogonal
directions are first-order nuisance directions. 
% \carlee{This is a nice result, but is it used anywhere? If not, how does it add to our argument?}

In particular, if \(T\) is any transformation on grounded transition
structures such that \(M(TG)=M(G)\), then
$
\bar Q^\star(o,M(TG),a)=\bar Q^\star(o,M(G),a).
$
If instead \(\{T_\varepsilon\}_{\varepsilon>0}\) is a family of
transformations such that
$
M(T_\varepsilon G)=M(G)+H_\varepsilon,
\qquad
H_\varepsilon\to 0,
$
and
$
\langle W_{o,M(G),a},H_\varepsilon\rangle=0
\qquad
\text{for all }\varepsilon,
$
then
$
\bar Q^\star(o,M(T_\varepsilon G),a)-\bar Q^\star(o,M(G),a)
=
o(\|H_\varepsilon\|_F).
$
No first-order claim is intended for discrete edits whose induced matrix
change is not small in Frobenius norm. 

\begin{proposition}[Population sufficiency of the matrix state for value regression]
\label{prop:matrix-regression-sufficiency}
Fix an environment \(e\) and any distribution over tuples
\((o_t,\tau_{0:t-1},a)\) induced by interaction with \(e\). Let
$
X:=(o_t,\tau_{0:t-1},a),
\qquad
Z:=(o_t,M(\tau_{0:t-1}),a),
\qquad
Y:=Q_e^\star(o_t,\tau_{0:t-1},a).
$
Under Assumption~\ref{ass:matrix-conditioned-value},
$
Y=\bar Q^\star(Z)
\qquad\text{almost surely}.
$
Hence
$
\mathbb E[Y\mid X]=\mathbb E[Y\mid Z]=\bar Q^\star(Z)
\qquad\text{almost surely},
$
and, for squared loss,
$
\inf_{f}\mathbb E\bigl[(Y-f(X))^2\bigr]
=
\inf_{g}\mathbb E\bigl[(Y-g(Z))^2\bigr]
=
0,
$
where the infima range over measurable predictors \(f\) and \(g\). Therefore, at the population level, access to the full transition history cannot improve optimal value prediction beyond the matrix state.
\end{proposition}

Proposition~\ref{prop:matrix-regression-sufficiency} is a population sufficiency statement, not a finite-sample generalization bound. Its role is to justify why one should prefer a matrix-conditioned value model to an unrestricted history-to-value map when the matrix-conditioned value law is the intended inductive bias.

\subsection{Proof of Proposition~\ref{prop:matrix-value-stability}}

\begin{proof}
Fix \(o\in\mathcal O\), \(a\in\mathcal A\), and a reference matrix \(M\). By Assumption~\ref{ass:matrix-value-smoothness}, there exist an open set
\[
\mathcal U_{o,a}\subseteq\mathbb S^m
\]
containing \(M\) and a function
\[
\widehat Q^\star_{o,a}:\mathcal U_{o,a}\to\mathbb R
\]
such that
\[
\widehat Q^\star_{o,a}(N)=\bar Q^\star(o,N,a)
\]
for every relevant \(N\in\mathcal U_{o,a}\). Moreover,
\(D\widehat Q^\star_{o,a}\) is locally bounded in the operator norm induced by the Frobenius norm.

Choose \(r>0\) such that the closed Frobenius ball
\[
\overline B_r(M):=\{N\in\mathbb S^m:\|N-M\|_F\le r\}
\]
is contained in \(\mathcal U_{o,a}\). For each \(N\in\overline B_r(M)\), local boundedness of \(D\widehat Q^\star_{o,a}\) gives an open neighborhood \(\mathcal V_N\subseteq\mathcal U_{o,a}\) and a constant \(L_N<\infty\) such that
\[
\|D\widehat Q^\star_{o,a}(P)\|_{\mathrm{op},F}\le L_N
\qquad
\text{for all }P\in\mathcal V_N,
\]
where \(\|\cdot\|_{\mathrm{op},F}\) denotes the operator norm induced by \(\|\cdot\|_F\). The family \(\{\mathcal V_N:N\in\overline B_r(M)\}\) is an open cover of the compact set \(\overline B_r(M)\), so it admits a finite subcover \(\mathcal V_{N_1},\dots,\mathcal V_{N_K}\). Let
\[
L_{o,M,a}:=\max_{1\le j\le K} L_{N_j}<\infty.
\]
Then
\[
\|D\widehat Q^\star_{o,a}(P)\|_{\mathrm{op},F}\le L_{o,M,a}
\qquad
\text{for all }P\in \overline B_r(M).
\]

Now let \(N_0,N_1\) lie in the relevant matrix domain and satisfy \(\|N_i-M\|_F<r\). Since Frobenius balls are convex, the segment
\[
N_\theta:=(1-\theta)N_0+\theta N_1,
\qquad
\theta\in[0,1],
\]
lies in \(\overline B_r(M)\). Define
\[
g(\theta):=\widehat Q^\star_{o,a}(N_\theta).
\]
Because \(\widehat Q^\star_{o,a}\) is Fr\'echet differentiable, \(g\) is differentiable on \([0,1]\) and the chain rule gives
\[
g'(\theta)=D\widehat Q^\star_{o,a}(N_\theta)[N_1-N_0].
\]
Therefore
\[
|g'(\theta)|
\le
\|D\widehat Q^\star_{o,a}(N_\theta)\|_{\mathrm{op},F}\,\|N_1-N_0\|_F
\le
L_{o,M,a}\|N_1-N_0\|_F.
\]
The fundamental theorem of calculus yields
\[
|\widehat Q^\star_{o,a}(N_1)-\widehat Q^\star_{o,a}(N_0)|
=
\left|\int_0^1 g'(\theta)\,d\theta\right|
\le
\int_0^1 |g'(\theta)|\,d\theta
\le
L_{o,M,a}\|N_1-N_0\|_F.
\]
Since \(\widehat Q^\star_{o,a}\) agrees with \(\bar Q^\star(o,\cdot,a)\) on the relevant domain, this proves
\[
|\bar Q^\star(o,N_1,a)-\bar Q^\star(o,N_0,a)|
\le
L_{o,M,a}\|N_1-N_0\|_F
\]
for all \(N_0,N_1\) in the domain with \(\|N_i-M\|_F<r\). This is the first claim, with \(\mathcal U:=\{N:\|N-M\|_F<r\}\) intersected with the relevant matrix domain.

For the second claim, Theorem~\ref{thm:matrix-bellman-linearization} gives
\[
\bar Q^\star(o,M+H,a)-\bar Q^\star(o,M,a)
=
\langle W_{o,M,a},H\rangle
+
o(\|H\|_F)
\]
as \(\|H\|_F\to0\). If \(\langle W_{o,M,a},H\rangle=0\), the leading first-order term vanishes and therefore
\[
\bar Q^\star(o,M+H,a)-\bar Q^\star(o,M,a)
=
o(\|H\|_F).
\]
This proves the stated first-order invisibility of value-null directions.
\end{proof}

\subsection{Proof of Proposition~\ref{prop:matrix-regression-sufficiency}}

\begin{proof}
Fix an environment \(e\) and any distribution over tuples \((o_t,\tau_{0:t-1},a)\) induced by interaction with \(e\). Define
\[
X:=(o_t,\tau_{0:t-1},a),
\qquad
Z:=(o_t,M(\tau_{0:t-1}),a),
\qquad
Y:=Q_e^\star(o_t,\tau_{0:t-1},a).
\]
Assumption~\ref{ass:matrix-conditioned-value} states pointwise that
\[
Y=\bar Q^\star(Z).
\]
Hence \(Y\) is a measurable function of \(Z\), and since \(Z\) itself is a measurable function of \(X\), the random variable \(Y\) is measurable with respect to both \(\sigma(Z)\) and \(\sigma(X)\). A basic property of conditional expectation is that if a random variable is measurable with respect to the conditioning sigma-algebra, then its conditional expectation given that sigma-algebra equals the variable itself almost surely. Therefore
\[
\mathbb E[Y\mid X]=Y
\qquad\text{almost surely}
\]
and
\[
\mathbb E[Y\mid Z]=Y
\qquad\text{almost surely}.
\]
Using \(Y=\bar Q^\star(Z)\), both identities become
\[
\mathbb E[Y\mid X]=\mathbb E[Y\mid Z]=\bar Q^\star(Z)
\qquad\text{almost surely}.
\]

Now consider squared loss. For any measurable predictor \(h\),
\[
(Y-h)^2\ge 0
\qquad\text{pointwise},
\]
so
\[
\mathbb E[(Y-h)^2]\ge 0.
\]
Choosing the history-based predictor
\[
f^\star(X):=\bar Q^\star(Z)=Y
\]
gives
\[
\mathbb E[(Y-f^\star(X))^2]=0.
\]
Hence
\[
\inf_f \mathbb E[(Y-f(X))^2]=0.
\]
Likewise, choosing the matrix-based predictor
\[
g^\star(Z):=\bar Q^\star(Z)=Y
\]
gives
\[
\mathbb E[(Y-g^\star(Z))^2]=0,
\]
and therefore
\[
\inf_g \mathbb E[(Y-g(Z))^2]=0.
\]
Combining the two equalities proves
\[
\inf_f \mathbb E[(Y-f(X))^2]
=
\inf_g \mathbb E[(Y-g(Z))^2]
=
0.
\]
Thus, at the population level, full history information offers no improvement over the matrix state for predicting the optimal value.
\end{proof}

\paragraph{Justification of the transformation paragraph.}
If \(M(TG)=M(G)\), then
\[
\bar Q^\star(o,M(TG),a)=\bar Q^\star(o,M(G),a)
\]
by direct substitution into the matrix-conditioned value law. For the family
\(\{T_\varepsilon\}_{\varepsilon>0}\), define
\[
H_\varepsilon:=M(T_\varepsilon G)-M(G).
\]
Proposition~\ref{prop:matrix-value-stability} gives
\[
\bar Q^\star(o,M(G)+H_\varepsilon,a)-\bar Q^\star(o,M(G),a)
=
\langle W_{o,M(G),a},H_\varepsilon\rangle
+
o(\|H_\varepsilon\|_F).
\]
If \(\langle W_{o,M(G),a},H_\varepsilon\rangle=0\) for all \(\varepsilon\), then
\[
\bar Q^\star(o,M(T_\varepsilon G),a)-\bar Q^\star(o,M(G),a)
=
o(\|H_\varepsilon\|_F),
\]
which is the claimed first-order invariance statement.

%%%%%%%%%%%%%%%%%%%%%%%%%%%%%%%%%%%%%%%%%%%%%%%%%%%%%%%%%%%%

\end{document}